%% file: paper-soft-locality.tex
\documentclass[reprint,aps,hypertext,pre]{revtex4-1}
\usepackage[utf8]{inputenc}
\usepackage[english]{babel}
\usepackage{amssymb}
\usepackage{amsmath}
\usepackage{amsthm}
\usepackage{natbib}
\usepackage{etoolbox}
\usepackage{tikz}
\usepackage{url}
\usepackage{hyperref}
\usepackage{pgfplots}

\newtheorem{theorem}{Theorem}
\newtheorem{definition}[theorem]{Definition}

\newcommand{\figref}[1]{Fig.~\ref{#1}}
\newcommand{\Figref}[1]{Fig.~\ref{#1}}
\newcommand{\defref}[1]{Definition~\ref{#1}}
\newcommand{\thmref}[1]{Theorem~\ref{#1}}
\newcommand{\secref}[1]{Section~\ref{#1}}

\newcommand{\RR}{\mathbb{R}}
\newcommand{\RRnonneg}{\RR_{\ge 0}}
\newcommand{\transpose}{{}^{T}} 

\newcommand{\qcpm}{q_\text{cpm}}
\newcommand{\Obj}{q}
\newcommand{\indicator}{\mathbf{1}}
\DeclareMathOperator{\symdiff}{\triangle}

\DeclareMathOperator*{\argmax}{argmax}
\DeclareMathOperator{\support}{supp}

\tikzset{
  node/.style={circle,draw,inner sep=0,minimum size=5mm},
  lbl/.style={fill=white,inner sep=0.2mm,minimum size=1mm},
}

\pgfplotsset{compat=newest} 

\begingroup
    \makeatletter
    \@for\theoremstyle:=definition,remark,plain\do{%
        \expandafter\g@addto@macro\csname th@\theoremstyle\endcsname{%
            \addtolength\thm@preskip\parskip
            }%
        }
\endgroup

\begin{document}

\title{Resolution-limit-free and local Non-negative Matrix Factorization quality functions for graph clustering}

\author{Twan van Laarhoven}\email{tvanlaarhoven@cs.ru.nl}
\author{Elena Marchiori}\email{elenam@cs.ru.nl}
\affiliation{Institute for Computing and Information Sciences, Radboud University Nijmegen, The Netherlands}

\begin{abstract}
Many graph clustering quality functions suffer from a resolution limit, the inability to find small clusters in large graphs.
So called resolution-limit-free quality functions do not have this limit. This property was previously introduced for hard clustering, that is, graph partitioning.

We investigate the resolution-limit-free property in the context of Non-negative Matrix Factorization (NMF) for hard and soft graph clustering. 
To use NMF in the hard clustering setting, a common approach is to assign each node to its highest membership cluster.
We show that in this case symmetric NMF is not resolution-limit-free,
 but that it becomes so when hardness constraints are used as part of the optimization. The resulting function is strongly linked to the Constant Potts
Model.
In soft clustering, nodes can belong to more than one cluster, with varying degrees of membership. In this setting resolution-limit-free turns out to be too strong a property. Therefore we introduce \emph{locality}, which roughly states that changing one part of the graph does not affect the clustering of other parts of the graph.
We argue that this is a desirable property, provide conditions under which NMF quality functions are local, and propose a novel class of local probabilistic NMF quality functions for soft graph clustering.

\end{abstract}

\maketitle

\section{Introduction}

Graph clustering, also known as network community detection,  is an important problem with  real-life applications in diverse disciplines such as life and social sciences \citep{Schaeffer2007,Fortunato2010}.  Graph clustering  is often performed by optimizing  a quality function, that is a function that assigns a score to a clustering. During the last decades, many such functions (and algorithms to optimize them) have been proposed. However, relatively little effort has been devoted to the theoretical foundation of graph clustering quality functions, e.g. \cite{AckermanBen-David2008axioms}.  In this paper we try to provide a contribution in this direction by studying desirable locality properties of quality functions for hard and soft graph clustering. 

We focus on the resolution-limit-free property, a  property of  hard graph clustering, recently introduced by \citet*{Traag2011ResolutionLimitScope}.
Resolution-limit-freeness is essentially a locality property. 
Informally this property states that a subset of an optimal clustering in the original graph should also be an optimal clustering in the induced subgraph containing only the nodes in the subset of clusters.
As the name suggests, resolution-limit-free quality functions do not suffer from the so-called resolution limit,  that is,  the inability to find small clusters in large graphs.   In the seminal work by \citet{Fortunato2007ResolutionLimit}, it was shown that modularity \citep{NewmanGirvan2004}, a popular quality function used for network community detection, has a resolution limit, in the sense that it may not detect clusters smaller than a scale which depends on the total size of the network and on the degree of interconnectedness of the clusters.  

Our goal is to investigate resolution-limit-freeness and other locality properties of Non-negative Matrix Factorization (NMF) graph clustering quality functions.
NMF \citep{Paatero1994,Lee1999} is a popular machine learning method initially used to  learn the parts of objects, like human faces and text documents. It finds two non-negative matrices whose product provides a good approximation to the input matrix. 
The non-negative constraints lead to a parts-based representation because they allow only additive, not subtractive, combinations.
Recently NMF formulations have been proposed as quality functions for graph clustering,  see for instance the surveys \citet{Wang2011} and \citet{LiDing2013nmfsurvey}.

We consider symmetric and asymmetric NMF formulations based on Euclidean loss and a Bayesian NMF quality function recently proposed by \citet{Psorakis2011NMF}, which can automatically determine the number of clusters.

The resolution-limit-free property is stated in the setting of hard clustering, where a clustering is a partition of the nodes. In contrast, NMF produces a soft clustering. Nodes have varying degrees of memberships of each clusters, and the clusters can overlap.
To use NMF in the hard clustering setting, a common approach is to assign each node to its highest membership cluster. 

%
In \secref{sec:hard}
we show that hard clustering based on NMF in this way is, in general, not resolution-limit-free.
For symmetric NMF we show that resolution-limit-freeness can be obtained by using orthogonality constraints as part of the optimization, and that the resulting function is strongly linked to the Constant Potts Model (CPM). CPM was introduced by \citeauthor{Traag2011ResolutionLimitScope} as the simplest formulation of a (non-trivial) resolution-limit-free method. It is a variant of the Potts model by \citet{Reichardt2004}.

We argue in \secref{sec:soft} that in the soft clustering setting, resolution-limit-freeness is a too strong property and propose an alternative desirable locality property  for soft graph clustering.
We characterize an interesting class of local quality functions and  show that symmetric and asymmetric NMF belong to this class. 
We show that Bayesian NMF is not local in general and that it suffers from a resolution limit. In \secref{sec:probabilistic} we introduce a novel class of probabilistic NMF quality functions that are local, and hence do not suffer from a resolution limit.

%
\subsection{Related work}
\label{sec:related-work}

The notion of resolution limit was introduced in \citet{Fortunato2007ResolutionLimit}, which detected a limitation of modularity, considered a state-of-the-art method for community detection.
\Citet{vanLaarhoven2013lso} showed empirically that the resolution limit is the most important difference between quality functions in graph clustering optimized using a fast local search algorithm, the Louvain method \citep{Blondel2008}. \Citet{Traag2011ResolutionLimitScope} introduced the notion of resolution-limit-free objective functions, which provides the motivation of this study.

Other local properties of quality functions for clustering have been considered in theoretical studies but mainly in the hard setting,  for distance based  clustering \citep{AckermanBenDavidLokerCOLT2010} and for graph clustering  \citep{vanLaarhoven2014axioms}.
Locality as defined in \citet{AckermanBenDavidLokerCOLT2010}  is a property of clustering functions, therein defined as functions mapping a data set and a positive integer $k$ to a partition of the data into $k$ clusters. This notion of locality was used together with other properties to characterize linkage based clustering.  The locality property considered in \citet{vanLaarhoven2014axioms} is part of an axiomatic study of quality functions for hard graph clustering.  It states that local changes to a graph should have only local consequences to a clustering. It is slightly weaker than the locality property considered in this study, which corresponds more closely to the property there called strong locality.

%

\subsection{Definitions and Notation}
\label{sec:definitions-and-notation}

A (weighted) \emph{graph} is a pair $(V,A)$ of a finite set $V$ of nodes and a function $A : V \times V \to \RRnonneg$ of edge weights. For compactness we view $A$ as an adjacency matrix, and write $a_{ij} = A(i,j)$.
Edges with larger weights represent stronger connections, so $a_{ij}=0$ means that there is no edge between nodes $i$ and $j$.

A graph $G'=(V',A')$ is a \emph{subgraph} of $G=(V,A)$
if $V' \subseteq V$ and $a'_{ij} = a_{ij}$ for all $i,j \in V'$.

Different clustering methods use different notions of a `cluster' and `clustering'.
For instance, in symmetric NMF a clustering is a matrix of membership coefficients;
while in non-symmetric NMF there are two such matrices. Some methods also have additional parameters for each cluster.

We therefore keep the notion of cluster abstract for now.
All we require is a function $\support$ from clusters to sets of nodes, called the \emph{support}.
We take a \emph{clustering} to be a set of clusters.
The support of a clustering $C$, written $\support(C)$, is the union of the support of all clusters in that clustering.
If the support of a clustering $C$ is a subset of a set $V$ of nodes, then we say that $C$ is a \emph{clustering of} $V$.
And for brevity, we also say that $C$ is a clustering of a graph $G$ if $C$ is a clustering of the nodes of $G$.

Note that this definition implies that if $C$ and $D$ are clusterings of $G$, then $C \cup D$ is also a clustering of $G$. And if $G$ is a subgraph of $G'$, then $C$ and $D$ are also clusterings of $G'$.

The symmetric difference of two clusterings is denoted $C \symdiff D$, and is defined as the symmetric difference of sets, that is $C \symdiff D = (C \cup D) \setminus (C \cap D)$.

Graph clustering can be cast as an optimization problem. The objective that is being optimized is the \emph{clustering quality function}, which is a function from graphs $G$ and clusterings of $G$ to real numbers. In this paper we take the convention that the quality is maximized.

Given a clustering quality function $q$, and a clustering $C$ of some graph $G$. We say that $C$ is \emph{$q$-optimal} if $q(G,C) \ge q(G,C')$ for all clusterings $C'$ of $G$.


\section{Non-negative Matrix Factorization}
\label{sec:nmf}

At its core, Non-negative Matrix Factorization  decomposes a matrix $A$ as a product $A \approx W H\transpose$, where all entries in $W$ and $H$ are non-negative.
For graph clustering the matrix $A$ is the adjacency matrix of a graph.
For undirected graphs the adjacency matrix is symmetric, in which case it makes sense to decompose it as $A \approx H H\transpose$. Since the optimal non-symmetric factorization of a symmetric matrix does not necessarily have $W=H$ \citep{Catral04}.

The columns of $W$ and $H$ can be interpreted as clusters.
To fit with the definition of `cluster' from the previous paragraph we need to take a slightly different view.
In the case of symmetric NMF, a cluster $c$ consists of a function $h_c$ from nodes to non-negative real numbers, called the membership coefficients. For a fixed set of nodes $h_c$ can be represented as a vector, and a set of such vectors can be seen as a matrix. The support of such a cluster $c$ is the support of the function $h_c$, that is, the set of nodes for which the membership is positive.
For non-symmetric NMF, a cluster is a tuple $c=(w_c,h_c)$ of two such functions $w_c$ and $h_c$; and for Bayesian NMF \citep{Psorakis2011NMF} each cluster also contains a $\beta_c$ parameter.

A common notion to all NMF methods is that they predict a value for each edge.
For symmetric NMF with per cluster membership $h_c$ this prediction can be written as
  $\hat{a}_{ij} = \sum_{c \in C}h_{ci} h_{cj}$.
For asymmetric NMF with cluster memberships $w_c$ and $h_c$ we can write
  $\hat{a}_{ij} = \sum_{c \in C}w_{ci} h_{cj}$.

The optimization problem then tries to ensure that $\hat{a}_{ij} \approx a_{ij}$. Different methods can have different interpretations of the `$\approx$' symbol, and they impose different regularizations and perhaps additional constraints.
Perhaps the simplest NMF quality function for undirected graphs uses Euclidean distance and no additional regularization,
\begin{equation*}
  q_\text{SymNMF}(G,C) = -\frac{1}{2} \sum_{i,j \in V} (a_{ij} - \hat{a}_{ij})^2.
\end{equation*}

%

%

\section{Resolution-limit-free functions for hard clustering}
\label{sec:hard}


Before we investigate the resolution limits of NMF,
we will first look at traditional `hard' clustering, where each node belongs to exactly one cluster.
In this setting a cluster is simply a subset of the nodes, and its support is the cluster itself.
There is the additional non-overlapping or orthogonality constraint on clusters: In a valid hard clustering $C$ of $V$, each node $i \in V$ is in exactly one cluster $c_i \in C$.
For symmetric NMF we may formulate these constraints as
\begin{align*}
  &\sum_{i \in V} h_{ci} h_{di} = 0 &&\text{ for all } c,d \in C, c\neq d, and \\
  &\sum_{c \in C} h_{ci} = 1 &&\text{ for all } i\in V.
\end{align*}

\Citet{Traag2011ResolutionLimitScope} introduced a locality property of clustering quality functions, and called the functions that satisfy this property \emph{resolution-limit-free}. 
Their definition is as follows.

\begin{definition}[Resolution-limit-free]
 \label{def:resolution-limit-free}
 Let $C$  be a $\Obj$-optimal clustering of a graph $G_1$. Then the quality function $\Obj$ is
called \emph{resolution-limit-free} if for each subgraph $G_2$ induced by $D \subset C$, the partition $D$ is a $\Obj$-optimal clustering of $G_2$.
\end{definition}

Thus in the setting of hard clustering, a quality function is resolution-limit-free if any subset of clusters from an optimal clustering is also an optimal clustering on the graph that contains only the nodes and edges in those clusters.

NMF has been extended with a post-processing step to yield a hard clustering. This is done by assigning each node to the cluster with the largest membership coefficient.

We can now ask if NMF with this post-processing is resolution-limit-free. In \figref{fig:counterexample-res-lim-free} we give a counterexample that answers this question negatively for the NMF based methods of \citet{Psorakis2011NMF} and \citet{Ding05onthe}.


This counterexample consists of two cliques and one almost-clique.
Additionally there is a node with unclear membership. When the entire graph is considered, its membership of one cluster is slightly higher, when one clique and its incident edges are removed, its membership of another cluster is slightly higher.
This difference is very small. For example, with \citeauthor{Ding05onthe}'s method
in the optimal clustering of the large graph, the disputed node belongs to the second and third cluster with membership coefficients $0.2306$ and $0.2311$ respectively; while in the smaller subgraph the membership coefficients are $0.2284$ and $0.2607$.

\newcommand{\bracket}[4][]{
  \draw[very thick,#1] (#2,#4+0.2) -- (#2,#4) -- (#3,#4) -- (#3,#4+0.2);
}
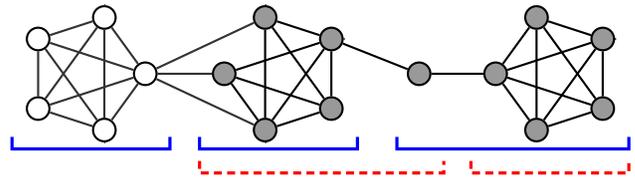
\begin{figure}
  \centering
  \def\circlescale{0.82}
  \newcommand{\circleclique}[5]{
    \begin{scope}[#3]
      \foreach \x in {1,...,#2} {
        \node (#1\x) [#4] at (\x*360/#2:1.2*\circlescale) {};
        \foreach \y in {1,...,#2} {
          \ifnum \x=\y \breakforeach \fi
          \ifnumequal{\x}{#5}{}{
            \draw[edge-#4] (#1\x) -- (#1\y);
          }
        }
      }
    \end{scope}
  }
  \begin{tikzpicture}
    [thick
    ,node/.style={circle,draw,inner sep=0,minimum size=3mm}
    ,ynode/.style={node,fill=black!40}
    ,edge-ynode/.style={}
    ,edge-node/.style={draw=black!80}
    ,scale=0.8
    ]
    \circleclique{a}{5}{shift={(0,0)},rotate=36*0/4}{node}{0}
    \circleclique{b}{5}{shift={(4*\circlescale,0)},rotate=36*4/4}{ynode}{2}
    \node (d) [ynode] at (6.75*\circlescale,0) {};
    \circleclique{c}{5}{shift={(9.5*\circlescale,0)},rotate=36*4/4}{ynode}{0}
    \draw[edge-node] (a5) -- (b2);
    \draw[edge-node] (a5) -- (b3);
    \draw[edge-node] (a5) -- (b1);
    \draw (b5) -- (d);
    \draw (c2) -- (d);
    \bracket[blue]{0*\circlescale-1.5*\circlescale}{0*\circlescale+1.7*\circlescale}{-1.6+0.35}
    \bracket[blue]{4*\circlescale-1.7*\circlescale}{4*\circlescale+1.5*\circlescale}{-1.6+0.35}
    \bracket[blue]{6.3*\circlescale}{9.5*\circlescale+1.5*\circlescale}{-1.6+0.35}
    \bracket[densely dashed,red]{4  *\circlescale-1.7*\circlescale}{7.25*\circlescale}{-2+0.35}
    \bracket[densely dashed,red]{9.5*\circlescale-1.7*\circlescale}{9.5*\circlescale+1.5*\circlescale}{-2+0.35}
  \end{tikzpicture}
  \caption{
    (Color online)
    A counterexample that shows that NMF quality functions are not resolution limit free.
    When considering the entire graph, the first (blue) clustering is optimal.
    When considering only the gray nodes, the second (dashed, red) clustering is optimal.
    The membership of the middle node is very unclear, it belongs to two clusters to almost the same degree. When another part of a cluster changes this can tip the balance one way or the other.
    %
  }
\label{fig:counterexample-res-lim-free}
\end{figure}


\Citet{Traag2011ResolutionLimitScope} showed that the Constant Potts model (CPM) is the simplest formulation
of any (non-trivial) resolution-limit-free method. The CPM quality function $\qcpm(G,C)$ can be formulated as


\begin{equation*}
  \qcpm(G,C) = \sum_{i,j \in V}(a_{ij} - \gamma)\indicator[c_i=c_j],
\end{equation*}
where $\indicator[c_i=c_j]$ is $1$ if nodes $i$ and $j$ belong to the same cluster, and $0$ otherwise.


Symmetric NMF and CPM are closely related. This can be shown with a technique similar to that used by \citet{Ding05onthe} to link symmetric NMF and spectral clustering.

\begin{theorem}
  Symmetric NMF is an instance of CPM with $\gamma = 1/2$ and  orthogonality constraints relaxed.
\end{theorem}
\begin{proof}
  \input{proof-cpm-snmf.tex}
\end{proof}



%

The CPM \emph{is} resolution-limit-free.
Therefore in order to perform hard clustering using symmetric NMF it is preferable to act on the quality function, for instance by enforcing orthogonality as done in \citep{Ding05onthe,Ding2006}, instead of assigning each node to the cluster with the highest membership coefficient.
%
%

%


\section{Resolution-limit-free functions for soft clustering}
\label{sec:soft}

The definition of resolution-limit-free hinges on the idea that, given some clusters in a graph, one can look at just the nodes that are in those clusters, without affecting other clusters.
In the hard clustering setting if a node is in some cluster then it is not in another cluster. So if we look only at nodes in clusters from a set $C$, then we automatically know that these nodes are not in clusters not in $C$.
When clusters can overlap this is no longer true.

We could still try to directly adapt \defref{def:resolution-limit-free} to the soft clustering setting, by defining what a graph induced by a subclustering is. The obvious idea is to include all nodes in the support of the subclustering.
So for a clustering $C$ of $G$, the graph $G'$ induced by $D \subseteq C$ would contain only the nodes which are in at least one cluster in $D$, that is, $V' = \support(D)$, and all edges between these nodes from the original graph.

However,  optimal clusterings might have clusters in $C\setminus D$ that overlap with clusters in $D$.
This makes this notion of resolution-limit-free too restrictive, and no NMF will satisfy it. 
But these NMF methods might still be resolution-limit-free in the sense that  optimal clusterings depend on the size of the entire graph.
So there should be a weaker notion that ensures that a quality function does not suffer from the problems associated with the resolution limit.

This makes this notion of resolution-limit-free too restrictive, since it effectively disallows any interesting uses of overlapping clusters. Consider the graph with two overlapping 5-cliques shown in \figref{fig:need-overlap}.
A reasonable clustering of this graph is the one with two overlapping clusters, $c_1$ and $c_2$, corresponding to the two cliques. In an NMF style method, the optimal membership coefficients of the two shared nodes (dark in the figure) for each of the clusters will be smaller than the membership of the other nodes to these clusters, since the edge between them is in both clusters. If the membership coefficients for all nodes were the same, then the prediction for the edge between the two nodes that are in both clusters would be twice as large as the prediction for other edges, which is not optimal. So the membership of these nodes in both clusters must be smaller than that of the other nodes.
But the subgraph induced by the left cluster $c_1$ is just a clique with 5 nodes. So the single cluster in the optimal clustering on this subgraph has equal membership for all nodes. In other words, the optimal clustering on the subgraph is not the same as $c_1$. Hence no NMF method is resolution-limit-free in this sense.

\begin{figure}
  \centering
  \def\gscale{1.5}
  \def\gbscale{0.2}
  \begin{tikzpicture}
    [thick
    ,node/.style={circle,draw,inner sep=0,minimum size=3mm}
    ,ynode/.style={node,fill=black!60}
    ,dnode/.style={node,fill=black!20}
    ,cnode/.style={circle,draw,inner sep=0,minimum size=3mm,dash pattern=on 2.5pt off 0.5pt}
    ,cedge/.style={dash pattern=on 2.5pt off 1pt}
    ,cbracket/.style={dash pattern=on 3pt off 1.2pt}
    ]
    \def\circlesize{5}
    \def\circlescale{1.2*0.9}
    \begin{scope}[]
      \foreach \x in {1,...,\circlesize} {
        \ifnum \x<3
          \node (a\x) [ynode] at (7*36+\x*360/\circlesize:\circlescale) {};
        \else
          \node (a\x) [dnode] at (7*36+\x*360/\circlesize:\circlescale) {};
        \fi
        \foreach \y in {1,...,5} {
          \ifnum \x=\y \breakforeach \fi
          \ifnum \x=2
            \draw [very thick] (a\x) -- (a\y);
          \else
            \draw (a\x) -- (a\y);
          \fi
        }
      }
      \begin{scope}[xshift=2cm*0.80902*\circlescale] 
        \foreach \x in {3,...,\circlesize} {
          \node (a\x) [dnode] at (2*36+\x*360/\circlesize:\circlescale) {};
          \foreach \y in {1,...,5} {
            \ifnum \x=\y \breakforeach \fi
            \draw (a\x) -- (a\y);
          }
        }
      \end{scope}
      \begin{scope}[yshift=-0.3cm]
        \bracket[black]{-1*\circlescale-\gbscale}{0.80902*\circlescale+\gbscale}{-1*\circlescale-.1}
        \bracket[black]{0.80902*\circlescale-\gbscale}{1*\circlescale+2*0.80902*\circlescale+\gbscale}{-1*\circlescale-.4}
    \node[left] at (-1*\circlescale-\gbscale,-1*\circlescale-.1+.05) {$c_1$};
    \node[right] at (3.05,-1*\circlescale-.4+.05) {$c_2$};
      \end{scope}
    \end{scope}
  \end{tikzpicture}
  \caption{Two cliques sharing two nodes and an edge (dark colored). The obvious clustering consists of two overlapping clusters.}
  \label{fig:need-overlap}
\end{figure}
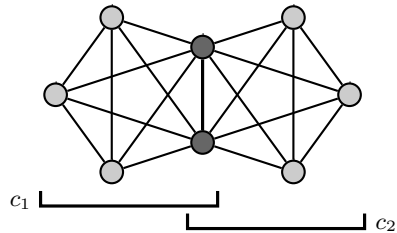

An alternative approach is to only consider subclusterings with disjoint support in the definition of resolution-limit-free.
That is, with $\support(D) \cap \support(C\setminus D) = \emptyset$.
Unfortunately this variant has the opposite problem: the condition almost never holds. So, many quality functions would trivially satisfy this variant of resolution-limit-freeness. For example, the optimal clusterings in NMF methods based on a Poisson likelihood will always have overlapping clusters covering every edge, so the disjointness condition only holds when the graph has multiple connected components.

Clearly we need a compromise.

\subsection{Locality}
\label{sec:locality}

The resolution-limit-free property looks at the behavior of a clustering quality function on graphs of different sizes. Intuitively a quality function suffers from a resolution limit if  optimal clusterings at a small scale depend on the size of the entire graph. 

As shown in the previous paragraph we can not just zoom in to the scale of any subclustering $D$ by discarding the rest of the graph.

But if we let go of only considering the optimal clustering, it does become possible to zoom in only partially, leaving the part of the graph covered by clusters that overlap clusters in $D$ intact. If $D$ is an optimal clustering of the original graph, then it should be a `locally optimal' clustering of the smaller graph in some sense. 

We take this to mean that if a clustering $D$ is better than some other clustering $D'$ on the original graph, then the same holds on the smaller graph, as long as $D$ and $D'$ induce the same zoomed in graph.

It then makes sense to not only consider zooming in by discarding the rest of the graph, but also consider arbitrary changes to the rest of the graph, as well as arbitrary changes to clusters not overlapping with $D$ or $D'$.

More precisely, if one subclustering $D$ is better than another subclustering $D'$ on a subgraph of some graph $G_1$, and one changes the graph to $G_2$ in such a way that the changes to the graph and to the clustering are disjoint from the support of $D$ and $D'$, then $D$ will stay a better clustering than $D'$.

This idea is illustrated in \figref{fig:locality}, and formalized in \defref{def:locality}
%

\begin{definition}[Locality]
  \label{def:locality}
  A clustering quality function $q$ is \emph{local} if
  for all graphs $G_1$, $G_2$, and common subgraphs $G_S$ of $G_1$ and $G_2$,  
  for all clusterings $C_1$ of $G_1$, $C_2$ of $G_2$, and $D,D'$ of $G_S$,
  such that $\support(C_1 \symdiff C_2) \cap \support(D \cup D') = \emptyset$,
  it is the case that
  $q(G_1,C_1 \cup D) \ge q(G_1,C_1 \cup D')$ if and only if
  $q(G_2,C_2 \cup D) \ge q(G_2,C_2 \cup D')$.
\end{definition}

\begin{figure}[t]
  \centering
  \def\gscale{1.1}
  \def\gbscale{0.22}
  \begin{tikzpicture}
    [thick
    ,node/.style={circle,draw,inner sep=0,minimum size=3mm}
    ,ynode/.style={node,fill=black!60}
    ,dnode/.style={node,fill=black!20}
    ,cnode/.style={circle,draw,inner sep=0,minimum size=3mm,dash pattern=on 2.5pt off 0.5pt}
    ,cedge/.style={dash pattern=on 2.5pt off 1pt}
    ,cbracket/.style={dash pattern=on 3pt off 1.2pt}
    ]
    \begin{scope}
      \node[ynode] (c1) at (0,0) {};
      \node[ynode] (c2) at (0,-1*\gscale) {};
      \node[ynode] (c3) at (1*\gscale,0) {};
      \node[ynode] (c4) at (1*\gscale,-1*\gscale) {};
      \node[dnode] (c5) at (2*\gscale,0) {};
      \node[dnode] (c6) at (2*\gscale,-1*\gscale) {};
      \node[cnode] (v1) at (2.8*\gscale,-0.5*\gscale) {};
      \draw (c1)--(c2);
      \draw (c1)--(c3);
      \draw (c2)--(c4);
      \draw (c4)--(c3);
      \draw[cedge] (c3)--(c5);
      \draw[cedge] (c4)--(c6);
      \draw[cedge] (c5)--(c6);
      \draw[cedge] (c5)--(v1);
      \draw[cedge] (c6)--(v1);
      \begin{scope}[yshift=-0.3cm]
        \bracket[blue]{-\gbscale}{0*\gscale+\gbscale}{-1.*\gscale-.1}
        \bracket[blue]{1*\gscale-\gbscale}{1*\gscale+\gbscale}{-1*\gscale-.1}
        \bracket[black]{1*\gscale-\gbscale}{2*\gscale+\gbscale}{-1*\gscale-.4}
        \bracket[black,cbracket]{2*\gscale-\gbscale}{2.8*\gscale+\gbscale}{-1*\gscale-.1}
      \end{scope}
      \begin{scope}[yshift=-1.2cm]
        \bracket[red]{-\gbscale}{1*\gscale+\gbscale}{-1.*\gscale-.1}
        \bracket[black]{1*\gscale-\gbscale}{2*\gscale+\gbscale}{-1*\gscale-.4}
        \bracket[black,cbracket]{2*\gscale-\gbscale}{2.8*\gscale+\gbscale}{-1*\gscale-.1}
      \end{scope}
    \end{scope}
    \begin{scope}[xshift=4.6cm]
      \node[ynode] (c1) at (0*\gscale,0*\gscale) {};
      \node[ynode] (c2) at (0*\gscale,-1*\gscale) {};
      \node[ynode] (c3) at (1*\gscale,0*\gscale) {};
      \node[ynode] (c4) at (1*\gscale,-1*\gscale) {};
      \node[dnode] (c5) at (2*\gscale,0*\gscale) {};
      \node[dnode] (c6) at (2*\gscale,-1*\gscale) {};
      \node[cnode] (v1) at (3*\gscale,0*\gscale) {};
      \node[cnode] (v2) at (3*\gscale,-1*\gscale) {};
      \draw (c1)--(c2);
      \draw (c1)--(c3);
      \draw (c2)--(c4);
      \draw (c4)--(c3);
      \draw[cedge] (c3)--(c6);
      \draw[cedge] (c4)--(c6);
      \draw[cedge] (c5)--(c6); 
      \draw[cedge] (c5)--(v1);
      \draw[cedge] (c6)--(v1);
      \draw[cedge] (v2)--(v1);
      \draw[cedge] (v2)--(c6);
      \begin{scope}[yshift=-0.3cm]
        \bracket[blue]{-\gbscale}{0*\gscale+\gbscale}{-1*\gscale-.1}
        \bracket[blue]{1*\gscale-\gbscale}{1*\gscale+\gbscale}{-1*\gscale-.1}
        \bracket[black]{1*\gscale-\gbscale}{2*\gscale+\gbscale}{-1*\gscale-.4}
        \bracket[black,cbracket]{2*\gscale-\gbscale}{3*\gscale+\gbscale}{-1*\gscale-.1}
        \bracket[black,cbracket]{3*\gscale-\gbscale}{3*\gscale+\gbscale}{-1*\gscale-.4}
      \end{scope}
      \begin{scope}[yshift=-1.2cm]
        \bracket[red]{-\gbscale}{1*\gscale+\gbscale}{-1*\gscale-.1}
        \bracket[black]{1*\gscale-\gbscale}{2*\gscale+\gbscale}{-1*\gscale-.4}
        \bracket[black]{1*\gscale-\gbscale}{2*\gscale+\gbscale}{-1*\gscale-.4}
        \bracket[black,cbracket]{2*\gscale-\gbscale}{3*\gscale+\gbscale}{-1*\gscale-.1}
        \bracket[black,cbracket]{3*\gscale-\gbscale}{3*\gscale+\gbscale}{-1*\gscale-.4}
      \end{scope}
    \end{scope}
  \end{tikzpicture}
  \caption{
    (Color online)
    An example illustrating locality.
    Between the left and right side, the dashed part of the clustering and the dashed part of the graph changes.
    The top and bottom clusterings differ only on the constant part (red/blue), and these differences don't overlap with changing clusters (dashed).
    Therefore if the top clustering has a higher quality than the bottom clustering on the left graph,
     then the same must hold on the right graph.
    Formally, the dark gray nodes are in the common subgraph $G_S$, the light gray nodes are in $\support(C_2 \cap C_1)$. The blue clustering is $D$, the red clustering $D'$, the solid black clusters is in both $C_1$ and $C_2$ and the dashed clusters are in only one of $C_1$ and $C_2$.
  }
  \label{fig:locality}
\end{figure}
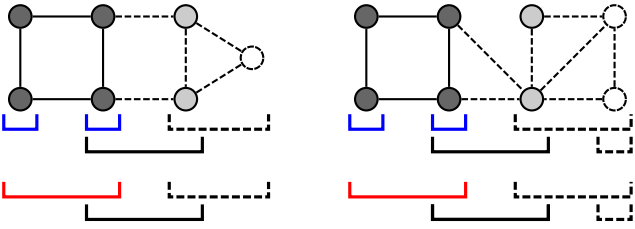

Locality as defined in \citet{AckermanBenDavidLokerCOLT2010} differs from our definition because it is a property of clustering functions.  The locality property considered in \citet{vanLaarhoven2014axioms} differs from our definition because it also enforces `agreement on the neighborhood' of the common subgraph. 
They also briefly discussed and dismissed a ``strong locality'' property, which is closer to our definition.



Even in the case of hard clustering locality and resolution-limit-free are not equivalent.
For hard clustering, locality implies resolution-limit-freeness, but the converse is not true.


\begin{theorem}
  If a hard clustering quality function is local, then it is resolution-limit-free.
\end{theorem}
\begin{proof}
  \input{proof-hard-strong-local-implies-resolution-free.tex}
\end{proof}

\begin{theorem}
  If a hard clustering quality function is resolution-limit-free then it is not necessarily local.
  \label{thm:res-lim-free-not-local}
\end{theorem}
\begin{proof}
  \input{proof-resolution-limit-free-not-implies-strong-local.tex}
\end{proof}

\subsection{Characterizing local quality functions}

Many quality functions can be written as a sum with a term for each edge, characterizing a goodness of fit, a term for each node, controlling the amount of overlap, and a term for each cluster, indicating some kind of complexity penalty. There might also be a constant term not actually depending on the clustering, and so not affecting the optimum. We call such quality functions additive.

\begin{definition}
  \label{def:additive}
  A qualify function is \emph{additive} if it can be written as
  \begin{align*}
    q(G,C) &= q_\text{graph}(G)
           + \sum_{c \in C} q_\text{clus}(c)
           \\&+ \sum_{i \in V} q_\text{node}\bigl(\{c \in C \mid i \in \support(c)\} \bigr)
           \\&+ \sum_{i \in V}\sum_{j \in V} q_\text{edge}\bigl( a_{ij}, \{c \in C \mid i,j \in \support(c)\} \bigr)
           \label{eq:ss}
  \end{align*}
  for some functions $q_\text{graph}$, $q_\text{clus}$, $q_\text{node}$, $q_\text{edge}$.
\end{definition}

Note that $q_\text{node}$ can depend on all clusters that contain node $i$, and $q_\text{edge}$ can depend on all clusters that contain the edge $ij$.


\begin{theorem}
  If a quality function is additive, then it is local.
  \label{thm:additive-local}
\end{theorem}
\begin{proof}
  \input{proof-additive-local2.tex}
\end{proof}
The converse of \thmref{thm:additive-local} does not hold; not all local quality functions are additive. For example, any monotonic function of a local quality function is also local.

Another example are quality functions that use higher order interactions, that is, it includes terms not only for nodes and edges, but also for triangles and larger structures. For instance, the clique percolation method \citep{Palla2005-clique-percolation} finds clusters which are cliques. That method is local, but it is not additive.
We could imagine including higher-order terms in the definition of additivity,
\begin{multline*}
q(G,C) = \dotsb +  \sum_{i,j,k \in V} q_\text{triangle}(a_{ij},a_{ik},a_{jk},
                  \\  \{ c \in C \mid i,j,k \in \support(c) \} ),
\end{multline*}
and so on. But for most purposes the edge term is sufficient; and the local quality functions that we consider in this paper are all additive in the sense of \defref{def:additive}.

Additivity provides additional insight into how quality functions behave: the quality is composed of the goodness-of-fit of a the clustering to nodes and edges (and perhaps larger structures), together with a cost term for each cluster.
By \thmref{thm:additive-local}, it also gives us a convenient way to \emph{prove} that a certain quality function is local, while locality can more convenient if we want to \emph{reason} about the behavior of a quality function.

For symmetric NMF, $\hat{a}_{ij}$ can be written as a sum over clusters that contain nodes $i$ and $j$,
\begin{equation*}
  \hat{a}_{ij} = \sum_{c \in C\text{ s.t. }i,j \in \support(c)} h_{ci} h_{cj}.
\end{equation*}
As a consequence, NMF quality functions without regularization, such as $q_\text{SymNMF}$, are additive.
Therefore these quality functions are local.



Many regularization terms can also be encoded in an additive quality function.
For example the L2 term $\sum_{c \in C}\sum_{i \in V}h_{ci}^2$ is a sum over clusters and independent of the graph, and so it fits in $q_\text{clus}$.

\subsection{Fixed number of clusters}
\label{sec:fixed}

The question of automatically finding the right number of clusters is still not fully solved.
Therefore in most NMF based clustering methods the number of clusters $k$ is specified by the user.

For most quality functions, if they are optimized directly without taking this restriction into account, then the number of clusters will tend to infinite. So we somehow need to fix the number of clusters.

The most direct way to incorporate this restriction of a fixed number of clusters is by adding it as a constraint to the quality function. That is, use $q(G,C,k) = q(G,C) + \indicator[|C| = k]\infty$.
Strictly speaking this is not a function to the real numbers. But we never need the fact that $q$ is such a function, all we need is that the quality of different clusterings can be compared.
Unfortunately, encoding a fixed $k$ restriction in the quality function violates locality.

Take two clusterings $C$ and $D$ of a graph $G$, with a different number of clusters.
Let $C'$, $D'$ and $G'$ be copies of $C$, $D$ and $G$ on a disjoint set of nodes, and let $k$ be $|C|+|D|$.
Then the quality $q(G\cup G',D\cup C',k)$ is finite, while $q(G \cup G', D \cup D',k)$ is infinite.
On the other hand, $q(G\cup G',C\cup C',k)$ is infinite, while $q(G\cup G',C\cup D',k)$ is finite.
This contradicts locality.

Instead, we need to consider the restriction on the number of clusters as separate from the quality function.
In that case the definition of locality can be used unchanged.

Equivalently, if we call a clustering consisting of $k$ clusters a $k$-clustering, then we can extend the definitions of locality to take the restricted number of clusters into account.
This approach is also used by \citet{Ackerman2013}.

If we call a function $q(G,C,k)$ for graphs $G$, clusterings $C$ and number of clusters $k$ a fixed-size quality function, then this leads to the following fixed-size variant of locality.

\begin{definition}[Fixed size locality]
  A fixed-size quality function $q$ is \emph{fixed-size local} if
  for all graphs $G_1$, $G_2$ and a common subgraph $G_S$, 
   for all $k_1$-clusterings $C_1$ of $G_1$, $k_2$-clusterings $C_2$ of $G_2$, and $m$-clustering $D$ of $G_S$ and $m'$-clusterings $D'$ of $G_S$,
  such that $\support(C_1 \symdiff C_2) \cap \support(D \cup D') = \emptyset$
  it is the case that
  $q(G_1,C_1 \cup D,k_1+m) \ge q(G_1,C_1 \cup D',k_1+m')$ if and only if
  $q(G_2,C_2 \cup D,k_2+m) \ge q(G_2,C_2 \cup D',k_2+m')$.
\end{definition}

Every local quality function that does not depend on $k$ is fixed-size local 
when combined with a constraint that the number of clusters must be $k$.
And so NMF with a fixed number of clusters is fixed-size local.

\subsection{Varying number of clusters}
\label{sec:varying}

\citet{Psorakis2011NMF} formulated a Bayesian formulation of NMF for overlapping community detection that uses automatic relevance determination (ARD) \citep{Tan09} to determine the number of clusters. Their quality functions can be written as
\begin{align*}
 q_\text{BayNMF} &= -\sum_{i \in V}\sum_{j \in V}\biggl( v_{ij} \log \frac{v_{ij}}{\hat{v}_{ij}}  + \hat{v}_{ij} \biggr)
                 \\& - \frac12 \sum_{c \in C} \biggl( \sum_{i \in V} \beta_c w_{ci}^2 + \sum_{i \in V} \beta_c h_{ci}^2 - 2 |V| \log \beta_c \biggr)
                 \\&- \sum_{c \in C} \bigl( \beta_c b - (a - 1)\log \beta_c \bigr)
                   - \kappa,
\end{align*}
where each cluster is a triple $c=(w_c,h_c,\beta_c)$, and $\kappa$ is a constant.
ARD works by fixing the number of clusters to some upper bound. In the optimal clustering many of these clusters $c$ will be empty, that is, have $\support(c) = \emptyset$.


This quality function is \emph{not} additive, for two reasons.
First of all, there is the term $2|V| \log \beta_c$ for each cluster, which stems from the half-normal priors on $W$ and $H$. This term depends on the number of nodes.
Secondly, the $\kappa$ term actually depends on the number of clusters and the number of nodes, since it contains the normalizing constants for the hyperprior on $\beta$, as well as constant factors for the half-normal priors.
For a fixed graph and fixed number of clusters the $\kappa$ term can be ignored, however.

As a result, \citeauthor{Psorakis2011NMF}'s method is also not local, as the following counterexample shows:

\begin{theorem}
  $q_\text{BayNMF}$ is not local.
\end{theorem}
\begin{proof}
Consider a graph $G_1$, consisting of a ring of $n=10$ cliques, where each clique has $m=5$ nodes, and two edges connecting it to the adjacent cliques.

We follow \citeauthor{Psorakis2011NMF}, and use hyperparameters $a = 5$ and $b = 2$.
This choice is not essential, similar counterexamples exist for other hyperparameter values.
As might be hoped, the $q_\text{BayNMF}$-optimal clustering $C_1$ of this graph then puts each clique in a separate cluster, with a small membership for the directly connected nodes in adjacent cliques.

This clustering is certainly better than the clustering $C_2$ with $5$ clusters each consisting of two cliques, and $5$ empty clusters.

However, on a larger graph with two disjoint copies of $G_1$, the clustering with two copies of $C_2$ is better than the clustering with two copies of $C_1$.

But by locality we would have $q_\text{BayNMF}(G_1\cup G_1', C_1\cup C_1') \ge q_\text{BayNMF}(G_1\cup G_1', C_2\cup C_1')$ as well as $q_\text{BayNMF}(G_1\cup G_1', C_2\cup C_1') \ge q_\text{BayNMF}(G_1\cup G_1', C_2\cup C_2')$, where the primed variables indicate copies with disjoint nodes. So $q_\text{BayNMF}$ is not local.
\end{proof}

In the above counterexample things don't change if one uses a ring of 20 cliques instead of two disjoint rings of 10 cliques.
This is closer to the original characterization of the resolution limit by \citet{Fortunato2007ResolutionLimit}.
In a ring of 20 cliques, the solution with 10 clusters is better than the solution with 20 clusters. But it is harder to show that this violates locality.

\section{NMF as a probabilistic model}
\label{sec:probabilistic}

NMF can be seen as a maximum likelihood fit of a generative probabilistic model.
The quality function that is optimized is then the log likelihood of the model conditioned on the observed graph,
\[ q(C,G) = \log P(C|G). \]

One  assumes that there is some underlying hidden cluster structure, and the edges in the graph depend on this structure.
The clustering structure in turn depends on the nodes under consideration.
So, by Bayes rule we may decompose $P(C|G)$ as
\[
  P(C|V,A) = P(A|C,V) P(C|V) P(V) / P(V,A).
\]
The terms $P(V)$ and $P(V,A)$ are constant given the graph, so the quality function becomes
\[ q(C,G) = \log P(A|C,V) + \log P(C|V) + \kappa, \]
where $\kappa = \log P(V) - \log P(V,A)$ is a constant.
The first term is the likelihood of the edges given the clustering, the second factor is the prior probability of a clustering for a certain set of nodes.

To make the above general formulation into an NMF model, one assumes that the edge weights are distributed independently, depending on the product of the membership matrices.
Then a prior is imposed on the membership coefficients. 
Usually a conjugate prior is used, which for Gaussian likelihood has a half-normal distribution, and for Poisson likelihood has a gamma distribution.
So the simplest symmetric Gaussian NMF method would be
\begin{align*}
  a_{ij} & \sim \mathcal{N}(\hat{a}_{ij},1) \\
  \hat{a}_{ij} & = \sum_{c} h_{ci} h_{cj} \\
  h_{ci} & \sim \mathcal{HN}(0,\sigma) \\
\end{align*}

Which leads to the quality function
\begin{align*}
  q(C,G) &= -\frac12 \sum_{i,j \in V} (a_{ij} - \hat{a}_{ij})^2
          - \frac1{2\sigma^2} \sum_{c \in C} \sum_{i \in V} h_{ci}^2
       \\&+ |V|^2 \log \sqrt{2 \pi}
          + |C||V| \log\sqrt{\pi \sigma^2/2}
          .
\end{align*}
This is a regularized variant of symmetric NMF discussed previously.

Such a model implicitly assumes a fixed number of clusters; and the corresponding quality function will not be local if the number of clusters is not fixed.
Intuitively, this happens because the model has to `pay' the normalizing constant of the prior distribution for each $h_{ci}$, the number of which is proportional to the number of clusters.

\citeauthor{Psorakis2011NMF} method also stems from a probabilistic model. They use a Poisson likelihood and a half-normal prior. Note that these are not conjugate. For finding the maximum likelihood solution conjugacy is not important. Using a conjugate prior becomes important only when doing variational Bayesian inference or Gibbs sampling \citep{Cemgil2009BayesianNMF}.

To determine the number of clusters, \citeauthor{Psorakis2011NMF} puts a gamma hyperprior on the inverse variance $\beta$. This allows a sharply peaked distribution on $w_c$ and $h_c$ when the support of a cluster is empty. The model is
\begin{align*}
  a_{ij} & \sim \text{Poisson}(\hat{a}_{ij}) \\
  \hat{a}_{ij} & = \sum_{c} h_{ci} w_{cj} \\
  h_{ci} & \sim \mathcal{HN}(0,1/\sqrt\beta_c) \\
  w_{ci} & \sim \mathcal{HN}(0,1/\sqrt\beta_c) \\
  \beta_c & \sim \text{Gamma}(a,b) \\
\end{align*}

As shown in \secref{sec:varying}, the corresponding quality function is not local.
The problems stem from the priors on $W$, $H$ and $\beta$, which depend on the number of nodes and clusters.
We will next try to find a different prior that is local.

\subsection{A local prior}

To get a local quality function from a probabilistic model, that doesn't assume a fixed number of clusters, we clearly need a different prior.
The approach we take will be to construct an additive quality function, which is local by \thmref{thm:additive-local}.




First assume as above that the likelihoods of the edges are independent and depending on the product of membership degrees,
that is $P(A|C,V) = \prod_{ij} P(a_{ij}|\hat{a}_{ij})$.
This fits nicely into the fourth term, $q_\text{edge}$, of an additive quality function.



Without loss of generality we can split the prior into two parts.
First the support of each cluster is determined, and based on this support the membership coefficients are chosen.
If we define $S=\{\support(c)|c \in C\}$\footnote{$S$ should be seen as a multiset, since multiple clusters can have the same support.}, then this means that
\[
  P(C|V) = P(C|V,S) P(S|V).
\]
A reasonable choice for the first term $P(C|V,S)$ is to assume that the clusters are independent, and that the membership coefficients inside each cluster are also independent, so
\begin{align*}
  C &= \{C_s \mid s \in S\} \\
  P(C_s|V,s) &= \prod_{c \in C} \Bigl( \prod_{i \in s} P(h_{ci}) \prod_{i \in V \setminus s} \delta(h_{ci},0) \Bigl)
  ,
\end{align*}
where $\delta$ is the Kronecker delta, which forces $h_{ci}$ to be zero for nodes not in $s$.
The logarithm of $P(C|V,S)$ is a sum of terms that depend only on a single cluster, so it can be encoded in the $q_\text{clus}$ term of an additive quality function.

Now consider $P(S|V)$.
If we know nothing about the nodes, then the two simplest aspects of $S$ we can look at are: (1) how many clusters cover each node, and (2) how many nodes are in each cluster.
The only local choice for (1) is to take the number of clusters that cover node $i$,  $n_i=\#\{s \in S \mid i \in s \}$, be independent and identically distributed according to some $f(n_i)$.
While for (2), the probability of a cluster $s \in S$ must be independent of the other clusters.
And since we have no information about the nodes, the only property of $s$ we can use is its size.
This suggests a prior of the form
\[
  P(S|V) = \frac{1}{Z} \prod_{i \in V} f(n_i)  \prod_{s\in S} g(|s|),
\]
where $n_i=|\{s \in S \mid i \in s \}|$ is the number of clusters covering node $i$.
The term $f(n_i)$ is local to each node, and can be encoded in $q_\text{node}$.
The term $g(|s|)$ is local to each cluster, and can therefore be encoded in $q_\text{clus}$.
The normalizing constant $Z$ depends only on $V$, and so it can be encoded in $q_\text{graph}$.

If we take $f(n_i) = \indicator[n_i=1]$ and $g(|s|)=(|s|-1)!$, then the prior on $S$ is exactly a Chinese Restaurant Process  \citep{Pitman2002}.
If we relax $f$, then we get a generalization where nodes can belong to multiple clusters.
Another choice is $f(n_i) = \indicator[n_i=1]$ and $g(|s|)=1$. Then the prior on $S$ is the flat prior over partitions, which is commonly used for hard clustering.

Yes another choice is to put a Poisson prior on either the number of clusters per node or the number of nodes per cluster. That is, take $f(n_i) = \lambda^{n_1}/n_1! \exp{-\lambda}$ for some constant $\lambda$, or do the same for $g$. This parameter allows the user to tune the number or size of clusters that are expected a-priori.

To summarize, we obtain a local quality function of the form
\begin{align*}
  q(G,C) &= \sum_{i \in V} \log f(|\{ c \in C \mid i \in \support(c) \}|)
      \\ &+ \sum_{c \in C} \log g(|\support(c)|)
      \\ &+ \sum_{c \in C} \sum_{i \in \support(c)} \log P(h_{ci})
      \\ &+ \sum_{i,j \in V} \log P(a_{ij} \mid \hat{a}_{ij})
      \\ &+ \kappa,
\end{align*}
which has four independent parts:
a score for a node being in a certain number of clusters,
a score for the size of each cluster,
a prior for each non-zero membership coefficient,
and the likelihood of an edge $a_{ij}$ given the $\hat{a}_{ij}$.



The discrete nature of this quality function makes it harder to optimize.
It is not clear if the multiplicative gradient algorithm that is commonly employed for NMF \citep{LeeSeung2000algorithmsForNMF}
can be adapted to deal with a prior on the support of clusters.
On the other hand, it might become possible to use discrete optimization methods, such as the successful Louvain method used for modularity maximization.

\subsection{Analysis of the quality functions on two types of graphs}

We will now investigate the local quality function proposed in the previous section.

First consider the original resolution limit model \citep{Fortunato2007ResolutionLimit}, which consists of a ring of cliques.
Two possible clusterings of a part of such a ring are illustrated in \figref{fig:ring-of-cliques}.

\begin{figure}
  \centering
  \begin{tikzpicture}
    [thick
    ,node/.style={circle,draw,inner sep=0,minimum size=3mm}
    ,ynode/.style={node,fill=black!20}
    ,nnode/.style={node,dotted}
    ]
    \def\circlescale{1.0}
    \def\circlesize{6}
    \begin{scope}[]
      \foreach \x in {1,...,\circlesize} {
        \node (a\x) [ynode] at (0.6*360/\circlesize+\x*360/\circlesize:\circlescale) {};
        \foreach \y in {1,...,\circlesize} {
          \ifnum \x=\y \breakforeach \fi
          \draw (a\x) -- (a\y);
        }
      }
    \end{scope}
    \begin{scope}[xshift=3cm]
      \foreach \x in {1,...,\circlesize} {
        \node (b\x) [ynode] at (0.4*360/\circlesize+\x*360/\circlesize:\circlescale) {};
        \foreach \y in {1,...,\circlesize} {
          \ifnum \x=\y \breakforeach \fi
          \draw (b\x) -- (b\y);
        }
      }
    \end{scope}
    \node [nnode] (c1) at (-3/1.5*\circlescale,-1.9/1.5*\circlescale) {};
    \node [nnode] (c2) at (3+3/1.5*\circlescale,-1.9/1.5*\circlescale) {};
    \draw (a4) -- (b4);
    \draw (a4) -- (c1);
    \draw (b4) -- (c2);
    
    \begin{scope}[yshift=0cm]
      \begin{scope}[rotate=0]
      \bracket[blue]{-1.1*\circlescale}{1.2*\circlescale}{-1.9}
      \end{scope}
      \begin{scope}[xshift=3cm]
      \bracket[blue]{-1.2*\circlescale}{+1.1*\circlescale}{-1.9}
      \end{scope}
      \bracket[blue]{-0.1*\circlescale}{3+0.1*\circlescale}{-2.2}
    \end{scope}
    \bracket[red]{-1.1*\circlescale}{3+1.1*\circlescale}{-2.8}
    \bracket[dotted]{-3/1.5*\circlescale-0.2}{0.2*\circlescale+0.2}{-3.5}
    \bracket[dotted]{3-0.2*\circlescale-0.2}{3+3/1.5*\circlescale+0.2}{-3.5}
    \node at (4.6,-2.0) {$D_1$};
    \node at (4.6,-2.75) {$D_2$};
    \node at (5.57,-3.40) {$C$};
  \end{tikzpicture}
  \caption{
    (Color online)
    Two possible clusterings in a subgraph of a ring of cliques.
    In the first clustering ($D_1$, blue), the two cliques are in separate clusters, and there is a third cluster for the edge between them.
    In the second clustering ($D_2$, red) two cliques are put into a single cluster.
    A third possibility is to include the middle edge in a cluster together with one of the two cliques.
    A clustering of this entire subgraph will also include two clusters covering the connecting edges ($C$, dotted).
  }
  \label{fig:ring-of-cliques}
\end{figure}

If a quality function is local, then we know that if $D_1 \cup C$ is a better clustering than $D_2 \cup C$ in this subgraph, then $D_1$ will also be better than $D_2$ as part of a larger graph. In other words, if the cliques are clustered correctly in a small ring, than this is true regardless of the number of cliques in the ring (unless a clustering with very large clusters is suddenly better).

We have performed experiments with the prior from the previous section, to see what the optimal clustering will be in practice.
We use a Poisson likelihood, a half normal prior on the supported membership coefficients (with precision $\beta=1$), a Poisson prior on the number of clusters-per-node (with $\lambda=1$) and a flat prior on the number of nodes per cluster.
To find the optimal clustering we use a general purpose optimization method, combined with a search over the possible supports of the clusters.

\Figref{fig:ring-of-cliques-results-num} shows that, as expected, the optimal solution is always to have one cluster per clique when using the local quality function. For comparison we also looked at the simpler non-local NMF method without a prior on the support. In that case the optimal solution depends strongly on the prior on membership coefficients $\beta$. If $\beta$ is small, then there is a penalty for every zero in the membership matrix, and hence a penalty on the number of clusters that increases with the number of nodes. If $\beta$ is large enough, then the probability density $p(0)>1$, and this penalty becomes a `bonus'. In that case adding even an empty cluster would improve the quality, and the optimal clustering has an infinite number of clusters.

The method of \citeauthor{Psorakis2011NMF} has the same resolution limit problem, but to an even larger extent. To automatically determine the number of clusters, this method keeps the actual number of clusters fixed to a large upper bound, for which the authors take the number of nodes. This means that there are very many clusters which will be empty in the optimal solution. For these empty clusters, the parameter $\beta_c$ becomes very large. And as said in the previous paragraph, this results in a bonus for empty clusters. Hence the method will tend to maximize the number of empty clusters, which results in a few large clusters actually containing the nodes.
For this experiment we used the prior $\beta_c \sim \text{Gamma}(5,2)$, as is also done in the code provided by \citeauthor{Psorakis2011NMF}.
Note that the jaggedness in the plot is due to the fact a ring of $n$ cliques can not always be divided evenly into $m$ clusters of equal size. Between 24 and 50 cliques, the optimal number of clusters is always 8 or 9.

\Figref{fig:ring-of-cliques-results-lambda} shows the influence of the parameter $\lambda$ of the Poisson prior that we put on the number of clusters per node.
When $\lambda$ becomes smaller, it becomes \emph{a priori} more likely for a node to be in only a single cluster, or in fact, to be in no cluster at all.
It actually requires a quite strong prior to get two cliques to merge into one cluster, when using $5$-cliques, we need $\lambda$ to be smaller than approximately $10^{-5}$. 

%
%
%

\pgfplotscreateplotcyclelist{objective comparison}{%
  red,every mark/.append style={fill=.!80!red},mark=*\\%
  green!60!black,every mark/.append style={fill=.!80!black},mark=triangle*\\%
  blue,every mark/.append style={fill=.!80!blue,scale=0.6},mark=square*\\%
  black,mark=star\\%
}
\pgfplotscreateplotcyclelist{objective comparison2}{%
  blue,every mark/.append style={fill=.!80!blue,scale=0.6},mark=square*\\%
}

\begin{figure}[t]
  \centering
  \begin{tikzpicture}
    \begin{axis}[
       xlabel={Number of cliques},ylabel={optimal cluster size},
       xmin=1,xmax=50,ymin=0.9,ymax=7.2,enlargelimits=false,
       legend cell align=left,
       cycle list name=objective comparison,
       ytick={1,2,3,4,5,6,7},
       width=0.88\linewidth,
       ]
     \pgfplotsset{every axis legend/.append style={at={(0.04,0.96)},anchor=north west}}
     \input{graph-ring-of-clique-psorakis.tex}
     \input{graph-ring-of-clique-lf00.tex} 
     \input{graph-ring-of-clique-lf1.tex} 
    \end{axis}
  \end{tikzpicture}
  \caption{
    (Color online)
    Optimal cluster size (average number of cliques per cluster) in a ring of $n$ $5$-cliques, when varying the number $n$ of cliques.
  }
  \label{fig:ring-of-cliques-results-num}
\end{figure}
\begin{figure}[t]
  \begin{tikzpicture}
    \begin{axis}[
       xlabel={$-\log_{10} \lambda$},ylabel={optimal cluster size},
       xmin=1,xmax=50,ymin=0.8,ymax=11.2,enlargelimits=false,
       cycle list name=objective comparison2,
       width=0.88\linewidth,
       ]
     \input{graph-ring-of-clique-lambda.tex} 
    \end{axis}
  \end{tikzpicture}
  \caption{
    (Color online)
    Optimal cluster size (average number of cliques per cluster) in a ring of $5$-cliques, when varying the $\lambda$ parameter of the Poisson prior on the number of clusters per node. The number of cliques in the ring doesn't matter because of locality.
  }
  \label{fig:ring-of-cliques-results-lambda}
\end{figure}


A ring of cliques is not a realistic model of real world graphs, since on most graphs the clustering is not as clear cut as it is there. The clustering problem can be made harder by removing edges inside the cliques, which are then no longer cliques, and better called modules; or by adding more edges between the modules.

We consider such a generalization, where there are two modules connected by zero or more edges. We then generated random modules and random between module edges. The two modules are either clustered together in one big cluster, or separated. In \figref{fig:realistic-results} we show simulation results of such a more realistic situation.
As we can see, as the number of between module edges increases, or the number of within module edges decreases, it becomes more likely to combine the two modules into one cluster. At the threshold between the two situations, the number of between module edges is roughly equal to the number of within module edges.
This matches the notion of a \emph{strong community}, which is defined by \citet{Radicchi2004} as a set of nodes having more edges inside the cluster than edges leaving the cluster.
A theoretical justification of these empirical results is beyond the scope of this work.

\begin{figure}[t]
  \centering
  \begin{tikzpicture}
    \begin{axis}[
       xlabel={Within module edges},
       ylabel={Between module edges},
       view={0}{90},
       colormap={my}{rgb(0cm)=(1,1,1); rgb(0.5cm)=(0.5,0.65,0.8); rgb(1cm)=(0,0.1,0.2)},
       width=0.88\linewidth,
       shader=flat corner,
       ymax=50.5,
       ]
       \input{plot-lf1-k10-mw-mb-one-or-three.tex}
    \end{axis}
  \end{tikzpicture}
  \caption{
    (Color online)
    Varying the number of within and between module edges.
    The modules each have $10$ nodes.
    A Poisson prior on the number of clusters per node ($\lambda=1$) was used.
    We consider two possible clusterings: (a) A solution with three clusters, two clusters for the two modules and one cluster for the between module edges. And (b) the solution with a single cluster containing all nodes.
    The color in the plot indicates which clustering has a higher quality. In the dark region, the clustering (a) with three clusters is better. In the light region, the solution (b) with a single cluster is better. Results are the average over 10 random graphs with the given number of edges.
  }
  \label{fig:realistic-results}
\end{figure}

\section{Conclusion}
To our knowledge, this work is the first to investigate  resolution-limit-free and local NMF quality functions for graph clustering. 
We gave a characterization of  a class of good (i.e. local) additive quality functions for graph clustering that provides a modular interpretation of NMF  for graph clustering. The definitions of locality and of additive quality functions are general, and can also be applied to other soft clustering methods. We proposed the class of  local probabilistic NMF quality functions.  The design and assessment of efficient algorithms for optimizing these quality functions remains to be investigated.


Results of this paper provide novel insights on NMF for hard clustering, 
on the resolution limit of Bayesian NMF for soft clustering,  and on the beneficial role of a local prior in  probabilistic formulations of NMF.
 




\section*{Acknowledgments}
This work has been partially funded by the Netherlands Organization for Scientific Research (NWO) within the NWO project 612.066.927.


\bibliography{clustering}

\end{document}

%% file: proof-cpm-snmf.tex
Recall that in symmetric NMF, $\hat{a}$ is defined as $\hat{a}_{ij} = \sum_{c \in C} h_{ci}h_{cj}$.
With orthogonality constraints, any two nodes $i$ and $j$ are either in the same cluster, in which case $\hat{a}_{ij}=1$, or they are in different clusters, in which case $\hat{a}_{ij}=0$.
So $\hat{a}_{ij} = \hat{a}_{ij}^2 = \indicator[c_i=c_j]$.

Symmetric NMF is given by the optimization problem
\begin{equation*}
  \argmax_{C} q_\text{SymNMF}(G,C) = -\frac12 \sum_{i,j \in V} (a_{ij} - \hat{a}_{ij})^2.
\end{equation*}
Expanding the square shows that this is equivalent to
\begin{equation*}
  \argmax_{C} \sum_{i,j \in V} \biggl(  a_{ij}\hat{a}_{ij} - \frac12 \hat{a}_{ij}^2 \biggr).
\end{equation*}
With orthogonality constraints this is equivalent to
\begin{equation*}
  \argmax_{C} \sum_{i,j \in V} \biggl(  a_{ij} - \frac12 \biggr)\hat{a}_{ij},
\end{equation*}
which is the CPM objective with $\gamma=1/2$.

%% file: proof-hard-strong-local-implies-resolution-free.tex
Let $\Obj$ be a local hard cluster quality function,
 and $C$ be a $\Obj$-optimal clustering of a graph $G_1=(V_1,A_1)$.
Consider the subgraph $G_2$ induced by $D \subset C$.



Let $C_1 = C\setminus D$ and $C_2 = \emptyset$, so $C_1 \symdiff C_2 = C_1$.
Because $C$ is a partition of $V_1$, we have for every clustering $D'$ of $G_2$ that $\support(C_1 \symdiff C_2) \cap \support(D \cup D') = \emptyset$.

Then for each clustering $D'$ of $G_2$ we have
 $q(G_1,C_1 \cup D) \ge q(G_1,C_1 \cup D')$ because $C_1 \cup D = C$ is an optimal clustering of $G_1$.
By locality it follows that
 $q(G_2,C_2 \cup D) \ge q(G_2,C_2 \cup D')$.

So $D$ is a $\Obj$-optimal clustering of $G_2$.

%% file: proof-resolution-limit-free-not-implies-strong-local.tex
  Consider the following quality function
  \begin{equation*}
    q(G,C) = \max_{c \in C} |c| + \min_{c \in C} |c|
  \end{equation*}
  For each graph $G=(V,A)$, the clustering $C=\{V\}$ is the single $q$-optimal clustering, with quality $2|V|$.
  Since there are no strict subsets of $C$ the quality function is trivially resolution-limit-free.
  
  Now consider the graphs $G_1$ with nodes $\{1,2,\dotsc,7\}$ and $G_2$ with nodes $\{1,2,\dotsc,6\}$, both with no edges.
  These graphs agree on the set $X=\{1,2,\dotsc,6\}$. Take the clusterings $E=\emptyset$, $D=\{\{1,2,3,4\},\{5\},\{6\}\}$, $D' = \{\{1,2,3\},\{5,6,7\}\}$, $C_1=\{\{7\}\}$ and $C_2=\{\}$. Then $q(G_1,C_1\cup D)=5>4=q(G_1,C_1\cup D')$, while    $q(G_2,C_2\cup D)=5<6=q(G_2,C_2\cup D')$.

  So $q$ is not local.
  
  This counterexample is illustrated in Figure~\ref{fig:counterexample-res-lim-free-local}.
\begin{figure}[h]
  \begin{center}
  \begin{tikzpicture}
    \def\scale{0.7}
    \def\yscale{0.6}
    \foreach \x in {1,2,...,7} {
      \node[node] at (\x*\scale,0) {\x};
    }
    \renewcommand{\bracket}[4][]{
      \draw[very thick,#1] (#2,#4+0.1) -- (#2,#4-0.1) -- (#3,#4-0.1) -- (#3,#4+0.1);
    }
    \bracket{0.6*\scale}{4.4*\scale}{-\yscale*1}
    \bracket{4.6*\scale}{5.4*\scale}{-\yscale*1}
    \bracket{5.6*\scale}{6.4*\scale}{-\yscale*1}
    \bracket{6.6*\scale}{7.4*\scale}{-\yscale*1}
    \bracket{0.6*\scale}{3.4*\scale}{-\yscale*2}
    \bracket{3.6*\scale}{6.4*\scale}{-\yscale*2}
    \bracket{6.6*\scale}{7.4*\scale}{-\yscale*2}
    \bracket{0.6*\scale}{4.4*\scale}{-\yscale*3}
    \bracket{4.6*\scale}{5.4*\scale}{-\yscale*3}
    \bracket{5.6*\scale}{6.4*\scale}{-\yscale*3}
    \bracket{0.6*\scale}{3.4*\scale}{-\yscale*4}
    \bracket{3.6*\scale}{6.4*\scale}{-\yscale*4}
    
    \node[right] at (8*\scale,-\yscale*1) {$q(G_1,C_1\cup D)=5$};
    \node[right] at (8*\scale,-\yscale*2) {$q(G_1,C_1\cup D')=4$};
    \node[right] at (8*\scale,-\yscale*3) {$q(G_2,C_2\cup D)=5$};
    \node[right] at (8*\scale,-\yscale*4) {$q(G_2,C_2\cup D')=6$};
  \end{tikzpicture}
  \end{center}
  \caption{The counterexample from the proof of Theorem~\ref{thm:res-lim-free-not-local}.}
  \label{fig:counterexample-res-lim-free-local}
\end{figure}

%% file: proof-additive-local2.tex

  Let $q$ be an additive quality function.
  Let $G_1$ $G_2$ and $G_S=(V,A)$ be graphs such that $G_S$ is a subgraph of both $G_1$ and $G_2$.

  Let $C_1$ be a clustering of $G_1$,
       $C_2$ a clustering of $G_2$ and,
       $D,D'$ clusterings of $G_S$
  such that $\support(C_1 \symdiff C_2) \cap \support(D \cup D') = \emptyset$.
  
  Let $E = C_1 \cap C_2$.
  Then for every node $i \in \support(C_1 \setminus C_2)$,
     we have
     $i \notin \support(D)$, so $\{ c \in C_1 \cup D \mid i \in \support(c) \} = \{ c \in C_1 \cup D' \mid i \in \support(c) \}  = \{ c \in C_1 \mid i \in \support(c) \}$.

  Conversely, for every node $i \notin \support(C_1 \setminus C_2)$,
     we have
     $\{ c \in C_1 \cup D \mid i \in \support(c) \} = \{ c \in E \cup D \mid i \in \support(c) \}$.

  Therefore
  \begin{align*}
    q(G_1&,C_1\cup D)-q(G_1,C_1\cup D') \\=&
    \sum_{c\in D} q_\text{clus}(c) - \sum_{c\in D'} q_\text{clus}(c)
             \\+& \sum_{i \in V} q_\text{node}\bigl(\{c \in E \cup D \mid i \in \support(c)\} \bigr)
             \\-& \sum_{i \in V} q_\text{node}\bigl(\{c \in E \cup D' \mid i \in \support(c)\} \bigr)
             \\+& \sum_{i,j \in V} q_\text{edge}\bigl( a_{ij}, \{c \in E \cup D \mid i,j \in \support(c)\} \bigr)
             \\-& \sum_{i,j \in V} q_\text{edge}\bigl( a_{ij}, \{c \in E \cup D' \mid i,j \in \support(c)\} \bigr)
             ,
  \end{align*}
  and similarly for $G_2$ and $C_2$ in place of the $G_1$ and $C_1$.
  
  Which implies that $q(G_1,C_1\cup D)-q(G_1,C_1\cup D') = q(G_2,C_2\cup D)-q(G_2,C_2\cup D')$.

  And so $q(G_1,C_1\cup D) \ge q(G_1,C_1 \cup D')$ if and only if $q(G_2,C_2\cup D) \ge q(G_2,C_2 \cup D')$.

  In other words, $q$ is local.
  

%% file: graph-ring-of-clique-psorakis.tex
\addplot+[] coordinates{
(2,1.)
(3,1.)
(4,1.)
(5,1.)
(6,1.)
(7,1.)
(8,1.)
(9,1.5)
(10,2.)
(11,1.83333)
(12,2.)
(13,1.85714)
(14,2.)
(15,1.875)
(16,2.)
(17,2.125)
(18,2.)
(19,2.71429)
(20,2.85714)
(21,3.)
(22,3.14286)
(23,2.875)
(24,3.)
(25,3.125)
(26,3.25)
(27,3.)
(28,3.5)
(29,3.625)
(30,3.75)
(31,3.875)
(32,4.)
(33,4.125)
(34,4.25)
(35,3.88889)
(36,4.)
(37,4.625)
(38,4.75)
(39,4.875)
(40,5.)
(41,5.125)
(42,5.25)
(43,4.77778)
(44,4.88889)
(45,5.)
(46,5.11111)
(47,5.875)
(48,6.)
(49,6.125)
(50,5.55556)

};
\addlegendentry{Psorakis}

%% file: graph-ring-of-clique-lf00.tex
\addplot+[] coordinates{
(1,1.)
(2,1.)
(3,1.)
(4,1.)
(5,1.)
(6,1.)
(7,1.)
(8,1.)
(9,1.)
(10,1.)
(11,1.)
(12,1.)
(13,1.)
(14,1.)
(15,1.)
(16,1.)
(17,1.)
(18,1.)
(19,1.)
(20,1.42857)
(21,1.5)
(22,1.57143)
(23,1.91667)
(24,2.)
(25,1.92308)
(26,2.)
(27,1.92857)
(28,2.)
(29,1.93333)
(30,2.)
(31,1.9375)
(32,2.)
(33,2.0625)
(34,2.)
(35,2.05882)
(36,2.)
(37,2.05556)
(38,2.)
(39,2.05263)
(40,2.5)
(41,2.5625)
(42,2.47059)
(43,2.86667)
(44,2.93333)
(45,3.)
(46,2.875)
(47,2.9375)
(48,3.)
(49,3.0625)
(50,2.94118)
};
\addlegendentry{no prior}

%% file: graph-ring-of-clique-lf1.tex
\addplot+[] coordinates{
(1,1.)
(2,1.)
(3,1.)
(4,1.)
(5,1.)
(6,1.)
(7,1.)
(8,1.)
(9,1.)
(10,1.)
(11,1.)
(12,1.)
(13,1.)
(14,1.)
(15,1.)
(16,1.)
(17,1.)
(18,1.)
(19,1.)
(20,1.)
(21,1.)
(22,1.)
(23,1.)
(24,1.)
(25,1.)
(26,1.)
(27,1.)
(28,1.)
(29,1.)
(30,1.)
(31,1.)
(32,1.)
(33,1.)
(34,1.)
(35,1.)
(36,1.)
(37,1.)
(38,1.)
(39,1.)
(40,1.)
(41,1.)
(42,1.)
(43,1.)
(44,1.)
(45,1.)
(46,1.)
(47,1.)
(48,1.)
(49,1.)
(50,1.)
};
\addlegendentry{Poisson prior, $\lambda=1$}

%% file: graph-ring-of-clique-lambda.tex
\addplot+[] coordinates{
(1,1.)
(2,1.)
(3,1.)
(4,1.)
(5,2.)
(6,2.)
(7,2.)
(8,2.)
(9,2.)
(10,3.)
(11,3.)
(12,3.)
(13,3.)
(14,3.)
(15,4.)
(16,4.)
(17,4.)
(18,4.)
(19,4.)
(20,5.)
(21,5.)
(22,5.)
(23,5.)
(24,5.)
(25,6.)
(26,6.)
(27,6.)
(28,6.)
(29,6.)
(30,7.)
(31,7.)
(32,7.)
(33,7.)
(34,7.)
(35,8.)
(36,8.)
(37,8.)
(38,8.)
(39,8.)
(40,9.)
(41,9.)
(42,9.)
(43,9.)
(44,10.)
(45,10.)
(46,10.)
(47,10.)
(48,10.)
(49,11.)
(50,11.)
};

%% file: plot-lf1-k10-mw-mb-one-or-three.tex
\addplot3[surf,draw=mapped color!80!gray] coordinates {(-0.5,-0.5,0.)
(-0.5,0.5,0.)
(-0.5,1.5,0.)
(-0.5,2.5,0.)
(-0.5,3.5,0.)
(-0.5,4.5,0.)
(-0.5,5.5,0.)
(-0.5,6.5,0.)
(-0.5,7.5,0.)
(-0.5,8.5,0.)
(-0.5,9.5,0.)
(-0.5,10.5,0.)
(-0.5,11.5,0.)
(-0.5,12.5,0.)
(-0.5,13.5,0.)
(-0.5,14.5,0.)
(-0.5,15.5,0.)
(-0.5,16.5,0.)
(-0.5,17.5,0.)
(-0.5,18.5,0.)
(-0.5,19.5,0.)
(-0.5,20.5,0.)
(-0.5,21.5,0.)
(-0.5,22.5,0.)
(-0.5,23.5,0.)
(-0.5,24.5,0.)
(-0.5,25.5,0.)
(-0.5,26.5,0.)
(-0.5,27.5,0.)
(-0.5,28.5,0.)
(-0.5,29.5,0.)
(-0.5,30.5,0.)
(-0.5,31.5,0.)
(-0.5,32.5,0.)
(-0.5,33.5,0.)
(-0.5,34.5,0.)
(-0.5,35.5,0.)
(-0.5,36.5,0.)
(-0.5,37.5,0.)
(-0.5,38.5,0.)
(-0.5,39.5,0.)
(-0.5,40.5,0.)
(-0.5,41.5,0.)
(-0.5,42.5,0.)
(-0.5,43.5,0.)
(-0.5,44.5,0.)
(-0.5,45.5,0.)
(-0.5,46.5,0.)
(-0.5,47.5,0.)
(-0.5,48.5,0.)
(-0.5,49.5,0.)
(-0.5,50.5,0)

(0.5,-0.5,1.)
(0.5,0.5,1.)
(0.5,1.5,0.9)
(0.5,2.5,0.7)
(0.5,3.5,0.5)
(0.5,4.5,0.3)
(0.5,5.5,0.2)
(0.5,6.5,0.)
(0.5,7.5,0.)
(0.5,8.5,0.)
(0.5,9.5,0.)
(0.5,10.5,0.)
(0.5,11.5,0.)
(0.5,12.5,0.)
(0.5,13.5,0.)
(0.5,14.5,0.)
(0.5,15.5,0.)
(0.5,16.5,0.)
(0.5,17.5,0.)
(0.5,18.5,0.)
(0.5,19.5,0.)
(0.5,20.5,0.)
(0.5,21.5,0.)
(0.5,22.5,0.)
(0.5,23.5,0.)
(0.5,24.5,0.)
(0.5,25.5,0.)
(0.5,26.5,0.)
(0.5,27.5,0.)
(0.5,28.5,0.)
(0.5,29.5,0.)
(0.5,30.5,0.)
(0.5,31.5,0.)
(0.5,32.5,0.)
(0.5,33.5,0.)
(0.5,34.5,0.)
(0.5,35.5,0.)
(0.5,36.5,0.)
(0.5,37.5,0.)
(0.5,38.5,0.)
(0.5,39.5,0.)
(0.5,40.5,0.)
(0.5,41.5,0.)
(0.5,42.5,0.)
(0.5,43.5,0.)
(0.5,44.5,0.)
(0.5,45.5,0.)
(0.5,46.5,0.)
(0.5,47.5,0.)
(0.5,48.5,0.)
(0.5,49.5,0.)
(0.5,50.5,0)

(1.5,-0.5,1.)
(1.5,0.5,1.)
(1.5,1.5,0.8)
(1.5,2.5,1.)
(1.5,3.5,0.7)
(1.5,4.5,0.5)
(1.5,5.5,0.6)
(1.5,6.5,0.2)
(1.5,7.5,0.4)
(1.5,8.5,0.1)
(1.5,9.5,0.)
(1.5,10.5,0.)
(1.5,11.5,0.)
(1.5,12.5,0.)
(1.5,13.5,0.)
(1.5,14.5,0.)
(1.5,15.5,0.)
(1.5,16.5,0.)
(1.5,17.5,0.)
(1.5,18.5,0.)
(1.5,19.5,0.)
(1.5,20.5,0.)
(1.5,21.5,0.)
(1.5,22.5,0.)
(1.5,23.5,0.)
(1.5,24.5,0.)
(1.5,25.5,0.)
(1.5,26.5,0.)
(1.5,27.5,0.)
(1.5,28.5,0.)
(1.5,29.5,0.)
(1.5,30.5,0.)
(1.5,31.5,0.)
(1.5,32.5,0.)
(1.5,33.5,0.)
(1.5,34.5,0.)
(1.5,35.5,0.)
(1.5,36.5,0.)
(1.5,37.5,0.)
(1.5,38.5,0.)
(1.5,39.5,0.)
(1.5,40.5,0.)
(1.5,41.5,0.)
(1.5,42.5,0.)
(1.5,43.5,0.)
(1.5,44.5,0.)
(1.5,45.5,0.)
(1.5,46.5,0.)
(1.5,47.5,0.)
(1.5,48.5,0.)
(1.5,49.5,0.)
(1.5,50.5,0)

(2.5,-0.5,1.)
(2.5,0.5,1.)
(2.5,1.5,1.)
(2.5,2.5,1.)
(2.5,3.5,1.)
(2.5,4.5,0.7)
(2.5,5.5,0.8)
(2.5,6.5,0.7)
(2.5,7.5,0.7)
(2.5,8.5,0.3)
(2.5,9.5,0.3)
(2.5,10.5,0.3)
(2.5,11.5,0.2)
(2.5,12.5,0.)
(2.5,13.5,0.1)
(2.5,14.5,0.1)
(2.5,15.5,0.)
(2.5,16.5,0.)
(2.5,17.5,0.)
(2.5,18.5,0.1)
(2.5,19.5,0.)
(2.5,20.5,0.)
(2.5,21.5,0.)
(2.5,22.5,0.)
(2.5,23.5,0.)
(2.5,24.5,0.)
(2.5,25.5,0.)
(2.5,26.5,0.)
(2.5,27.5,0.)
(2.5,28.5,0.)
(2.5,29.5,0.)
(2.5,30.5,0.)
(2.5,31.5,0.)
(2.5,32.5,0.)
(2.5,33.5,0.)
(2.5,34.5,0.)
(2.5,35.5,0.)
(2.5,36.5,0.)
(2.5,37.5,0.)
(2.5,38.5,0.)
(2.5,39.5,0.)
(2.5,40.5,0.)
(2.5,41.5,0.)
(2.5,42.5,0.)
(2.5,43.5,0.)
(2.5,44.5,0.)
(2.5,45.5,0.)
(2.5,46.5,0.)
(2.5,47.5,0.)
(2.5,48.5,0.)
(2.5,49.5,0.)
(2.5,50.5,0)

(3.5,-0.5,1.)
(3.5,0.5,1.)
(3.5,1.5,1.)
(3.5,2.5,1.)
(3.5,3.5,1.)
(3.5,4.5,1.)
(3.5,5.5,1.)
(3.5,6.5,0.7)
(3.5,7.5,0.6)
(3.5,8.5,0.7)
(3.5,9.5,0.4)
(3.5,10.5,0.2)
(3.5,11.5,0.3)
(3.5,12.5,0.3)
(3.5,13.5,0.2)
(3.5,14.5,0.1)
(3.5,15.5,0.1)
(3.5,16.5,0.1)
(3.5,17.5,0.)
(3.5,18.5,0.)
(3.5,19.5,0.)
(3.5,20.5,0.)
(3.5,21.5,0.)
(3.5,22.5,0.)
(3.5,23.5,0.)
(3.5,24.5,0.)
(3.5,25.5,0.)
(3.5,26.5,0.)
(3.5,27.5,0.)
(3.5,28.5,0.)
(3.5,29.5,0.)
(3.5,30.5,0.)
(3.5,31.5,0.)
(3.5,32.5,0.)
(3.5,33.5,0.)
(3.5,34.5,0.)
(3.5,35.5,0.)
(3.5,36.5,0.)
(3.5,37.5,0.)
(3.5,38.5,0.)
(3.5,39.5,0.)
(3.5,40.5,0.)
(3.5,41.5,0.)
(3.5,42.5,0.)
(3.5,43.5,0.)
(3.5,44.5,0.)
(3.5,45.5,0.)
(3.5,46.5,0.)
(3.5,47.5,0.)
(3.5,48.5,0.)
(3.5,49.5,0.)
(3.5,50.5,0)

(4.5,-0.5,1.)
(4.5,0.5,1.)
(4.5,1.5,1.)
(4.5,2.5,1.)
(4.5,3.5,1.)
(4.5,4.5,1.)
(4.5,5.5,0.9)
(4.5,6.5,1.)
(4.5,7.5,0.5)
(4.5,8.5,0.7)
(4.5,9.5,0.5)
(4.5,10.5,0.5)
(4.5,11.5,0.2)
(4.5,12.5,0.2)
(4.5,13.5,0.)
(4.5,14.5,0.)
(4.5,15.5,0.2)
(4.5,16.5,0.)
(4.5,17.5,0.)
(4.5,18.5,0.1)
(4.5,19.5,0.)
(4.5,20.5,0.)
(4.5,21.5,0.)
(4.5,22.5,0.)
(4.5,23.5,0.)
(4.5,24.5,0.)
(4.5,25.5,0.)
(4.5,26.5,0.)
(4.5,27.5,0.)
(4.5,28.5,0.)
(4.5,29.5,0.)
(4.5,30.5,0.)
(4.5,31.5,0.)
(4.5,32.5,0.)
(4.5,33.5,0.)
(4.5,34.5,0.)
(4.5,35.5,0.)
(4.5,36.5,0.)
(4.5,37.5,0.)
(4.5,38.5,0.)
(4.5,39.5,0.)
(4.5,40.5,0.)
(4.5,41.5,0.)
(4.5,42.5,0.)
(4.5,43.5,0.)
(4.5,44.5,0.)
(4.5,45.5,0.)
(4.5,46.5,0.)
(4.5,47.5,0.)
(4.5,48.5,0.)
(4.5,49.5,0.)
(4.5,50.5,0)

(5.5,-0.5,1.)
(5.5,0.5,1.)
(5.5,1.5,1.)
(5.5,2.5,1.)
(5.5,3.5,1.)
(5.5,4.5,1.)
(5.5,5.5,1.)
(5.5,6.5,1.)
(5.5,7.5,0.8)
(5.5,8.5,0.8)
(5.5,9.5,0.8)
(5.5,10.5,0.5)
(5.5,11.5,0.5)
(5.5,12.5,0.4)
(5.5,13.5,0.3)
(5.5,14.5,0.1)
(5.5,15.5,0.1)
(5.5,16.5,0.2)
(5.5,17.5,0.)
(5.5,18.5,0.)
(5.5,19.5,0.)
(5.5,20.5,0.1)
(5.5,21.5,0.)
(5.5,22.5,0.)
(5.5,23.5,0.)
(5.5,24.5,0.)
(5.5,25.5,0.)
(5.5,26.5,0.)
(5.5,27.5,0.)
(5.5,28.5,0.)
(5.5,29.5,0.)
(5.5,30.5,0.)
(5.5,31.5,0.)
(5.5,32.5,0.)
(5.5,33.5,0.)
(5.5,34.5,0.)
(5.5,35.5,0.)
(5.5,36.5,0.)
(5.5,37.5,0.)
(5.5,38.5,0.)
(5.5,39.5,0.)
(5.5,40.5,0.)
(5.5,41.5,0.)
(5.5,42.5,0.)
(5.5,43.5,0.)
(5.5,44.5,0.)
(5.5,45.5,0.)
(5.5,46.5,0.)
(5.5,47.5,0.)
(5.5,48.5,0.)
(5.5,49.5,0.)
(5.5,50.5,0)

(6.5,-0.5,1.)
(6.5,0.5,1.)
(6.5,1.5,1.)
(6.5,2.5,1.)
(6.5,3.5,1.)
(6.5,4.5,1.)
(6.5,5.5,1.)
(6.5,6.5,1.)
(6.5,7.5,0.9)
(6.5,8.5,0.9)
(6.5,9.5,0.7)
(6.5,10.5,0.7)
(6.5,11.5,0.8)
(6.5,12.5,0.4)
(6.5,13.5,0.2)
(6.5,14.5,0.3)
(6.5,15.5,0.2)
(6.5,16.5,0.1)
(6.5,17.5,0.)
(6.5,18.5,0.)
(6.5,19.5,0.1)
(6.5,20.5,0.)
(6.5,21.5,0.)
(6.5,22.5,0.1)
(6.5,23.5,0.)
(6.5,24.5,0.)
(6.5,25.5,0.)
(6.5,26.5,0.)
(6.5,27.5,0.)
(6.5,28.5,0.)
(6.5,29.5,0.)
(6.5,30.5,0.)
(6.5,31.5,0.)
(6.5,32.5,0.)
(6.5,33.5,0.)
(6.5,34.5,0.)
(6.5,35.5,0.)
(6.5,36.5,0.)
(6.5,37.5,0.)
(6.5,38.5,0.)
(6.5,39.5,0.)
(6.5,40.5,0.)
(6.5,41.5,0.)
(6.5,42.5,0.)
(6.5,43.5,0.)
(6.5,44.5,0.)
(6.5,45.5,0.)
(6.5,46.5,0.)
(6.5,47.5,0.)
(6.5,48.5,0.)
(6.5,49.5,0.)
(6.5,50.5,0)

(7.5,-0.5,1.)
(7.5,0.5,1.)
(7.5,1.5,1.)
(7.5,2.5,1.)
(7.5,3.5,1.)
(7.5,4.5,1.)
(7.5,5.5,1.)
(7.5,6.5,1.)
(7.5,7.5,0.9)
(7.5,8.5,1.)
(7.5,9.5,1.)
(7.5,10.5,0.7)
(7.5,11.5,0.6)
(7.5,12.5,0.4)
(7.5,13.5,0.4)
(7.5,14.5,0.5)
(7.5,15.5,0.1)
(7.5,16.5,0.3)
(7.5,17.5,0.2)
(7.5,18.5,0.2)
(7.5,19.5,0.)
(7.5,20.5,0.)
(7.5,21.5,0.)
(7.5,22.5,0.1)
(7.5,23.5,0.)
(7.5,24.5,0.)
(7.5,25.5,0.)
(7.5,26.5,0.)
(7.5,27.5,0.)
(7.5,28.5,0.)
(7.5,29.5,0.)
(7.5,30.5,0.)
(7.5,31.5,0.)
(7.5,32.5,0.)
(7.5,33.5,0.)
(7.5,34.5,0.)
(7.5,35.5,0.)
(7.5,36.5,0.)
(7.5,37.5,0.)
(7.5,38.5,0.)
(7.5,39.5,0.)
(7.5,40.5,0.)
(7.5,41.5,0.)
(7.5,42.5,0.)
(7.5,43.5,0.)
(7.5,44.5,0.)
(7.5,45.5,0.)
(7.5,46.5,0.)
(7.5,47.5,0.)
(7.5,48.5,0.)
(7.5,49.5,0.)
(7.5,50.5,0)

(8.5,-0.5,1.)
(8.5,0.5,1.)
(8.5,1.5,1.)
(8.5,2.5,1.)
(8.5,3.5,1.)
(8.5,4.5,1.)
(8.5,5.5,1.)
(8.5,6.5,1.)
(8.5,7.5,1.)
(8.5,8.5,0.9)
(8.5,9.5,1.)
(8.5,10.5,0.9)
(8.5,11.5,0.6)
(8.5,12.5,0.9)
(8.5,13.5,0.5)
(8.5,14.5,0.5)
(8.5,15.5,0.4)
(8.5,16.5,0.2)
(8.5,17.5,0.3)
(8.5,18.5,0.1)
(8.5,19.5,0.)
(8.5,20.5,0.1)
(8.5,21.5,0.)
(8.5,22.5,0.)
(8.5,23.5,0.1)
(8.5,24.5,0.1)
(8.5,25.5,0.)
(8.5,26.5,0.)
(8.5,27.5,0.)
(8.5,28.5,0.)
(8.5,29.5,0.)
(8.5,30.5,0.)
(8.5,31.5,0.)
(8.5,32.5,0.)
(8.5,33.5,0.)
(8.5,34.5,0.)
(8.5,35.5,0.)
(8.5,36.5,0.)
(8.5,37.5,0.)
(8.5,38.5,0.)
(8.5,39.5,0.)
(8.5,40.5,0.)
(8.5,41.5,0.)
(8.5,42.5,0.)
(8.5,43.5,0.)
(8.5,44.5,0.)
(8.5,45.5,0.)
(8.5,46.5,0.)
(8.5,47.5,0.)
(8.5,48.5,0.)
(8.5,49.5,0.)
(8.5,50.5,0)

(9.5,-0.5,1.)
(9.5,0.5,1.)
(9.5,1.5,1.)
(9.5,2.5,1.)
(9.5,3.5,1.)
(9.5,4.5,1.)
(9.5,5.5,1.)
(9.5,6.5,1.)
(9.5,7.5,1.)
(9.5,8.5,1.)
(9.5,9.5,1.)
(9.5,10.5,1.)
(9.5,11.5,0.9)
(9.5,12.5,0.7)
(9.5,13.5,0.5)
(9.5,14.5,0.6)
(9.5,15.5,0.7)
(9.5,16.5,0.5)
(9.5,17.5,0.5)
(9.5,18.5,0.)
(9.5,19.5,0.2)
(9.5,20.5,0.2)
(9.5,21.5,0.)
(9.5,22.5,0.1)
(9.5,23.5,0.2)
(9.5,24.5,0.)
(9.5,25.5,0.)
(9.5,26.5,0.)
(9.5,27.5,0.1)
(9.5,28.5,0.)
(9.5,29.5,0.)
(9.5,30.5,0.)
(9.5,31.5,0.)
(9.5,32.5,0.)
(9.5,33.5,0.)
(9.5,34.5,0.)
(9.5,35.5,0.)
(9.5,36.5,0.)
(9.5,37.5,0.)
(9.5,38.5,0.)
(9.5,39.5,0.)
(9.5,40.5,0.)
(9.5,41.5,0.)
(9.5,42.5,0.)
(9.5,43.5,0.)
(9.5,44.5,0.)
(9.5,45.5,0.)
(9.5,46.5,0.)
(9.5,47.5,0.)
(9.5,48.5,0.)
(9.5,49.5,0.)
(9.5,50.5,0)

(10.5,-0.5,1.)
(10.5,0.5,1.)
(10.5,1.5,1.)
(10.5,2.5,1.)
(10.5,3.5,1.)
(10.5,4.5,1.)
(10.5,5.5,1.)
(10.5,6.5,1.)
(10.5,7.5,1.)
(10.5,8.5,1.)
(10.5,9.5,1.)
(10.5,10.5,1.)
(10.5,11.5,1.)
(10.5,12.5,0.9)
(10.5,13.5,0.8)
(10.5,14.5,1.)
(10.5,15.5,0.5)
(10.5,16.5,0.6)
(10.5,17.5,0.3)
(10.5,18.5,0.6)
(10.5,19.5,0.1)
(10.5,20.5,0.1)
(10.5,21.5,0.)
(10.5,22.5,0.)
(10.5,23.5,0.)
(10.5,24.5,0.1)
(10.5,25.5,0.)
(10.5,26.5,0.)
(10.5,27.5,0.)
(10.5,28.5,0.)
(10.5,29.5,0.)
(10.5,30.5,0.)
(10.5,31.5,0.)
(10.5,32.5,0.)
(10.5,33.5,0.)
(10.5,34.5,0.)
(10.5,35.5,0.)
(10.5,36.5,0.)
(10.5,37.5,0.)
(10.5,38.5,0.)
(10.5,39.5,0.)
(10.5,40.5,0.)
(10.5,41.5,0.)
(10.5,42.5,0.)
(10.5,43.5,0.)
(10.5,44.5,0.)
(10.5,45.5,0.)
(10.5,46.5,0.)
(10.5,47.5,0.)
(10.5,48.5,0.)
(10.5,49.5,0.)
(10.5,50.5,0)

(11.5,-0.5,1.)
(11.5,0.5,1.)
(11.5,1.5,1.)
(11.5,2.5,1.)
(11.5,3.5,1.)
(11.5,4.5,1.)
(11.5,5.5,1.)
(11.5,6.5,1.)
(11.5,7.5,1.)
(11.5,8.5,1.)
(11.5,9.5,1.)
(11.5,10.5,1.)
(11.5,11.5,1.)
(11.5,12.5,1.)
(11.5,13.5,0.7)
(11.5,14.5,0.9)
(11.5,15.5,0.7)
(11.5,16.5,0.3)
(11.5,17.5,0.8)
(11.5,18.5,0.3)
(11.5,19.5,0.5)
(11.5,20.5,0.2)
(11.5,21.5,0.1)
(11.5,22.5,0.1)
(11.5,23.5,0.)
(11.5,24.5,0.1)
(11.5,25.5,0.)
(11.5,26.5,0.)
(11.5,27.5,0.)
(11.5,28.5,0.)
(11.5,29.5,0.)
(11.5,30.5,0.)
(11.5,31.5,0.)
(11.5,32.5,0.)
(11.5,33.5,0.)
(11.5,34.5,0.)
(11.5,35.5,0.)
(11.5,36.5,0.)
(11.5,37.5,0.)
(11.5,38.5,0.)
(11.5,39.5,0.)
(11.5,40.5,0.)
(11.5,41.5,0.)
(11.5,42.5,0.)
(11.5,43.5,0.)
(11.5,44.5,0.)
(11.5,45.5,0.)
(11.5,46.5,0.)
(11.5,47.5,0.)
(11.5,48.5,0.)
(11.5,49.5,0.)
(11.5,50.5,0)

(12.5,-0.5,1.)
(12.5,0.5,1.)
(12.5,1.5,1.)
(12.5,2.5,1.)
(12.5,3.5,1.)
(12.5,4.5,1.)
(12.5,5.5,1.)
(12.5,6.5,1.)
(12.5,7.5,1.)
(12.5,8.5,1.)
(12.5,9.5,1.)
(12.5,10.5,1.)
(12.5,11.5,0.9)
(12.5,12.5,0.9)
(12.5,13.5,1.)
(12.5,14.5,1.)
(12.5,15.5,0.7)
(12.5,16.5,0.4)
(12.5,17.5,0.2)
(12.5,18.5,0.5)
(12.5,19.5,0.1)
(12.5,20.5,0.3)
(12.5,21.5,0.2)
(12.5,22.5,0.2)
(12.5,23.5,0.)
(12.5,24.5,0.)
(12.5,25.5,0.)
(12.5,26.5,0.)
(12.5,27.5,0.)
(12.5,28.5,0.)
(12.5,29.5,0.)
(12.5,30.5,0.)
(12.5,31.5,0.)
(12.5,32.5,0.)
(12.5,33.5,0.)
(12.5,34.5,0.)
(12.5,35.5,0.)
(12.5,36.5,0.)
(12.5,37.5,0.)
(12.5,38.5,0.)
(12.5,39.5,0.)
(12.5,40.5,0.)
(12.5,41.5,0.)
(12.5,42.5,0.)
(12.5,43.5,0.)
(12.5,44.5,0.)
(12.5,45.5,0.)
(12.5,46.5,0.)
(12.5,47.5,0.)
(12.5,48.5,0.)
(12.5,49.5,0.)
(12.5,50.5,0)

(13.5,-0.5,1.)
(13.5,0.5,1.)
(13.5,1.5,1.)
(13.5,2.5,1.)
(13.5,3.5,1.)
(13.5,4.5,1.)
(13.5,5.5,1.)
(13.5,6.5,1.)
(13.5,7.5,1.)
(13.5,8.5,1.)
(13.5,9.5,1.)
(13.5,10.5,1.)
(13.5,11.5,1.)
(13.5,12.5,1.)
(13.5,13.5,0.7)
(13.5,14.5,0.9)
(13.5,15.5,0.7)
(13.5,16.5,0.7)
(13.5,17.5,0.4)
(13.5,18.5,0.6)
(13.5,19.5,0.4)
(13.5,20.5,0.3)
(13.5,21.5,0.3)
(13.5,22.5,0.2)
(13.5,23.5,0.1)
(13.5,24.5,0.1)
(13.5,25.5,0.1)
(13.5,26.5,0.1)
(13.5,27.5,0.1)
(13.5,28.5,0.)
(13.5,29.5,0.)
(13.5,30.5,0.)
(13.5,31.5,0.)
(13.5,32.5,0.)
(13.5,33.5,0.)
(13.5,34.5,0.)
(13.5,35.5,0.)
(13.5,36.5,0.)
(13.5,37.5,0.)
(13.5,38.5,0.)
(13.5,39.5,0.)
(13.5,40.5,0.)
(13.5,41.5,0.)
(13.5,42.5,0.)
(13.5,43.5,0.)
(13.5,44.5,0.)
(13.5,45.5,0.)
(13.5,46.5,0.)
(13.5,47.5,0.)
(13.5,48.5,0.)
(13.5,49.5,0.)
(13.5,50.5,0)

(14.5,-0.5,1.)
(14.5,0.5,1.)
(14.5,1.5,1.)
(14.5,2.5,1.)
(14.5,3.5,1.)
(14.5,4.5,1.)
(14.5,5.5,1.)
(14.5,6.5,1.)
(14.5,7.5,1.)
(14.5,8.5,1.)
(14.5,9.5,1.)
(14.5,10.5,1.)
(14.5,11.5,1.)
(14.5,12.5,1.)
(14.5,13.5,0.9)
(14.5,14.5,0.9)
(14.5,15.5,0.9)
(14.5,16.5,0.6)
(14.5,17.5,0.7)
(14.5,18.5,0.6)
(14.5,19.5,0.3)
(14.5,20.5,0.4)
(14.5,21.5,0.4)
(14.5,22.5,0.2)
(14.5,23.5,0.2)
(14.5,24.5,0.2)
(14.5,25.5,0.1)
(14.5,26.5,0.)
(14.5,27.5,0.)
(14.5,28.5,0.1)
(14.5,29.5,0.)
(14.5,30.5,0.)
(14.5,31.5,0.)
(14.5,32.5,0.)
(14.5,33.5,0.)
(14.5,34.5,0.)
(14.5,35.5,0.)
(14.5,36.5,0.)
(14.5,37.5,0.)
(14.5,38.5,0.)
(14.5,39.5,0.)
(14.5,40.5,0.)
(14.5,41.5,0.)
(14.5,42.5,0.)
(14.5,43.5,0.)
(14.5,44.5,0.)
(14.5,45.5,0.)
(14.5,46.5,0.)
(14.5,47.5,0.)
(14.5,48.5,0.)
(14.5,49.5,0.)
(14.5,50.5,0)

(15.5,-0.5,1.)
(15.5,0.5,1.)
(15.5,1.5,1.)
(15.5,2.5,1.)
(15.5,3.5,1.)
(15.5,4.5,1.)
(15.5,5.5,1.)
(15.5,6.5,1.)
(15.5,7.5,1.)
(15.5,8.5,1.)
(15.5,9.5,1.)
(15.5,10.5,1.)
(15.5,11.5,1.)
(15.5,12.5,1.)
(15.5,13.5,1.)
(15.5,14.5,1.)
(15.5,15.5,1.)
(15.5,16.5,0.6)
(15.5,17.5,0.7)
(15.5,18.5,0.5)
(15.5,19.5,0.6)
(15.5,20.5,0.2)
(15.5,21.5,0.5)
(15.5,22.5,0.2)
(15.5,23.5,0.2)
(15.5,24.5,0.2)
(15.5,25.5,0.)
(15.5,26.5,0.3)
(15.5,27.5,0.)
(15.5,28.5,0.1)
(15.5,29.5,0.1)
(15.5,30.5,0.)
(15.5,31.5,0.)
(15.5,32.5,0.)
(15.5,33.5,0.)
(15.5,34.5,0.)
(15.5,35.5,0.)
(15.5,36.5,0.)
(15.5,37.5,0.)
(15.5,38.5,0.)
(15.5,39.5,0.)
(15.5,40.5,0.)
(15.5,41.5,0.)
(15.5,42.5,0.)
(15.5,43.5,0.)
(15.5,44.5,0.)
(15.5,45.5,0.)
(15.5,46.5,0.)
(15.5,47.5,0.)
(15.5,48.5,0.)
(15.5,49.5,0.)
(15.5,50.5,0)

(16.5,-0.5,1.)
(16.5,0.5,1.)
(16.5,1.5,1.)
(16.5,2.5,1.)
(16.5,3.5,1.)
(16.5,4.5,1.)
(16.5,5.5,1.)
(16.5,6.5,1.)
(16.5,7.5,1.)
(16.5,8.5,1.)
(16.5,9.5,1.)
(16.5,10.5,1.)
(16.5,11.5,1.)
(16.5,12.5,1.)
(16.5,13.5,1.)
(16.5,14.5,1.)
(16.5,15.5,0.9)
(16.5,16.5,0.9)
(16.5,17.5,0.6)
(16.5,18.5,1.)
(16.5,19.5,0.6)
(16.5,20.5,0.5)
(16.5,21.5,0.5)
(16.5,22.5,0.3)
(16.5,23.5,0.4)
(16.5,24.5,0.2)
(16.5,25.5,0.2)
(16.5,26.5,0.)
(16.5,27.5,0.1)
(16.5,28.5,0.)
(16.5,29.5,0.)
(16.5,30.5,0.1)
(16.5,31.5,0.)
(16.5,32.5,0.)
(16.5,33.5,0.)
(16.5,34.5,0.)
(16.5,35.5,0.)
(16.5,36.5,0.)
(16.5,37.5,0.)
(16.5,38.5,0.)
(16.5,39.5,0.)
(16.5,40.5,0.)
(16.5,41.5,0.)
(16.5,42.5,0.)
(16.5,43.5,0.)
(16.5,44.5,0.)
(16.5,45.5,0.)
(16.5,46.5,0.)
(16.5,47.5,0.)
(16.5,48.5,0.)
(16.5,49.5,0.)
(16.5,50.5,0)

(17.5,-0.5,1.)
(17.5,0.5,1.)
(17.5,1.5,1.)
(17.5,2.5,1.)
(17.5,3.5,1.)
(17.5,4.5,1.)
(17.5,5.5,1.)
(17.5,6.5,1.)
(17.5,7.5,1.)
(17.5,8.5,1.)
(17.5,9.5,1.)
(17.5,10.5,1.)
(17.5,11.5,1.)
(17.5,12.5,1.)
(17.5,13.5,1.)
(17.5,14.5,1.)
(17.5,15.5,1.)
(17.5,16.5,0.9)
(17.5,17.5,0.9)
(17.5,18.5,1.)
(17.5,19.5,0.5)
(17.5,20.5,0.6)
(17.5,21.5,0.2)
(17.5,22.5,0.3)
(17.5,23.5,0.4)
(17.5,24.5,0.2)
(17.5,25.5,0.2)
(17.5,26.5,0.1)
(17.5,27.5,0.1)
(17.5,28.5,0.)
(17.5,29.5,0.1)
(17.5,30.5,0.2)
(17.5,31.5,0.)
(17.5,32.5,0.)
(17.5,33.5,0.)
(17.5,34.5,0.)
(17.5,35.5,0.)
(17.5,36.5,0.)
(17.5,37.5,0.)
(17.5,38.5,0.)
(17.5,39.5,0.)
(17.5,40.5,0.)
(17.5,41.5,0.)
(17.5,42.5,0.)
(17.5,43.5,0.)
(17.5,44.5,0.)
(17.5,45.5,0.)
(17.5,46.5,0.)
(17.5,47.5,0.)
(17.5,48.5,0.)
(17.5,49.5,0.)
(17.5,50.5,0)

(18.5,-0.5,1.)
(18.5,0.5,1.)
(18.5,1.5,1.)
(18.5,2.5,1.)
(18.5,3.5,1.)
(18.5,4.5,1.)
(18.5,5.5,1.)
(18.5,6.5,1.)
(18.5,7.5,1.)
(18.5,8.5,1.)
(18.5,9.5,1.)
(18.5,10.5,1.)
(18.5,11.5,1.)
(18.5,12.5,1.)
(18.5,13.5,1.)
(18.5,14.5,1.)
(18.5,15.5,0.9)
(18.5,16.5,1.)
(18.5,17.5,1.)
(18.5,18.5,0.9)
(18.5,19.5,0.8)
(18.5,20.5,0.8)
(18.5,21.5,0.4)
(18.5,22.5,0.6)
(18.5,23.5,0.5)
(18.5,24.5,0.3)
(18.5,25.5,0.)
(18.5,26.5,0.1)
(18.5,27.5,0.1)
(18.5,28.5,0.1)
(18.5,29.5,0.)
(18.5,30.5,0.1)
(18.5,31.5,0.2)
(18.5,32.5,0.)
(18.5,33.5,0.)
(18.5,34.5,0.)
(18.5,35.5,0.)
(18.5,36.5,0.)
(18.5,37.5,0.)
(18.5,38.5,0.)
(18.5,39.5,0.)
(18.5,40.5,0.)
(18.5,41.5,0.)
(18.5,42.5,0.)
(18.5,43.5,0.)
(18.5,44.5,0.)
(18.5,45.5,0.)
(18.5,46.5,0.)
(18.5,47.5,0.)
(18.5,48.5,0.)
(18.5,49.5,0.)
(18.5,50.5,0)

(19.5,-0.5,1.)
(19.5,0.5,1.)
(19.5,1.5,1.)
(19.5,2.5,1.)
(19.5,3.5,1.)
(19.5,4.5,1.)
(19.5,5.5,1.)
(19.5,6.5,1.)
(19.5,7.5,1.)
(19.5,8.5,1.)
(19.5,9.5,1.)
(19.5,10.5,1.)
(19.5,11.5,1.)
(19.5,12.5,1.)
(19.5,13.5,1.)
(19.5,14.5,1.)
(19.5,15.5,0.9)
(19.5,16.5,1.)
(19.5,17.5,0.8)
(19.5,18.5,1.)
(19.5,19.5,0.9)
(19.5,20.5,0.6)
(19.5,21.5,0.9)
(19.5,22.5,0.4)
(19.5,23.5,0.4)
(19.5,24.5,0.3)
(19.5,25.5,0.2)
(19.5,26.5,0.4)
(19.5,27.5,0.1)
(19.5,28.5,0.2)
(19.5,29.5,0.)
(19.5,30.5,0.)
(19.5,31.5,0.)
(19.5,32.5,0.)
(19.5,33.5,0.)
(19.5,34.5,0.)
(19.5,35.5,0.)
(19.5,36.5,0.)
(19.5,37.5,0.)
(19.5,38.5,0.)
(19.5,39.5,0.)
(19.5,40.5,0.)
(19.5,41.5,0.)
(19.5,42.5,0.)
(19.5,43.5,0.)
(19.5,44.5,0.)
(19.5,45.5,0.)
(19.5,46.5,0.)
(19.5,47.5,0.)
(19.5,48.5,0.)
(19.5,49.5,0.)
(19.5,50.5,0)

(20.5,-0.5,1.)
(20.5,0.5,1.)
(20.5,1.5,1.)
(20.5,2.5,1.)
(20.5,3.5,1.)
(20.5,4.5,1.)
(20.5,5.5,1.)
(20.5,6.5,1.)
(20.5,7.5,1.)
(20.5,8.5,1.)
(20.5,9.5,1.)
(20.5,10.5,1.)
(20.5,11.5,1.)
(20.5,12.5,1.)
(20.5,13.5,1.)
(20.5,14.5,1.)
(20.5,15.5,1.)
(20.5,16.5,1.)
(20.5,17.5,1.)
(20.5,18.5,1.)
(20.5,19.5,0.8)
(20.5,20.5,0.7)
(20.5,21.5,0.6)
(20.5,22.5,0.6)
(20.5,23.5,0.3)
(20.5,24.5,0.4)
(20.5,25.5,0.3)
(20.5,26.5,0.5)
(20.5,27.5,0.1)
(20.5,28.5,0.1)
(20.5,29.5,0.)
(20.5,30.5,0.)
(20.5,31.5,0.)
(20.5,32.5,0.)
(20.5,33.5,0.)
(20.5,34.5,0.)
(20.5,35.5,0.)
(20.5,36.5,0.)
(20.5,37.5,0.)
(20.5,38.5,0.)
(20.5,39.5,0.)
(20.5,40.5,0.)
(20.5,41.5,0.)
(20.5,42.5,0.)
(20.5,43.5,0.)
(20.5,44.5,0.)
(20.5,45.5,0.)
(20.5,46.5,0.)
(20.5,47.5,0.)
(20.5,48.5,0.)
(20.5,49.5,0.)
(20.5,50.5,0)

(21.5,-0.5,1.)
(21.5,0.5,1.)
(21.5,1.5,1.)
(21.5,2.5,1.)
(21.5,3.5,1.)
(21.5,4.5,1.)
(21.5,5.5,1.)
(21.5,6.5,1.)
(21.5,7.5,1.)
(21.5,8.5,1.)
(21.5,9.5,1.)
(21.5,10.5,1.)
(21.5,11.5,1.)
(21.5,12.5,1.)
(21.5,13.5,1.)
(21.5,14.5,1.)
(21.5,15.5,1.)
(21.5,16.5,0.9)
(21.5,17.5,1.)
(21.5,18.5,1.)
(21.5,19.5,1.)
(21.5,20.5,0.8)
(21.5,21.5,0.9)
(21.5,22.5,0.8)
(21.5,23.5,0.5)
(21.5,24.5,0.5)
(21.5,25.5,0.2)
(21.5,26.5,0.2)
(21.5,27.5,0.2)
(21.5,28.5,0.1)
(21.5,29.5,0.1)
(21.5,30.5,0.1)
(21.5,31.5,0.1)
(21.5,32.5,0.)
(21.5,33.5,0.)
(21.5,34.5,0.)
(21.5,35.5,0.)
(21.5,36.5,0.)
(21.5,37.5,0.)
(21.5,38.5,0.)
(21.5,39.5,0.)
(21.5,40.5,0.)
(21.5,41.5,0.)
(21.5,42.5,0.)
(21.5,43.5,0.)
(21.5,44.5,0.)
(21.5,45.5,0.)
(21.5,46.5,0.)
(21.5,47.5,0.)
(21.5,48.5,0.)
(21.5,49.5,0.)
(21.5,50.5,0)

(22.5,-0.5,1.)
(22.5,0.5,1.)
(22.5,1.5,1.)
(22.5,2.5,1.)
(22.5,3.5,1.)
(22.5,4.5,1.)
(22.5,5.5,1.)
(22.5,6.5,1.)
(22.5,7.5,1.)
(22.5,8.5,1.)
(22.5,9.5,1.)
(22.5,10.5,1.)
(22.5,11.5,1.)
(22.5,12.5,1.)
(22.5,13.5,1.)
(22.5,14.5,1.)
(22.5,15.5,1.)
(22.5,16.5,1.)
(22.5,17.5,1.)
(22.5,18.5,1.)
(22.5,19.5,1.)
(22.5,20.5,1.)
(22.5,21.5,0.9)
(22.5,22.5,0.9)
(22.5,23.5,0.7)
(22.5,24.5,0.6)
(22.5,25.5,0.7)
(22.5,26.5,0.2)
(22.5,27.5,0.2)
(22.5,28.5,0.2)
(22.5,29.5,0.)
(22.5,30.5,0.)
(22.5,31.5,0.2)
(22.5,32.5,0.)
(22.5,33.5,0.)
(22.5,34.5,0.)
(22.5,35.5,0.)
(22.5,36.5,0.)
(22.5,37.5,0.)
(22.5,38.5,0.)
(22.5,39.5,0.)
(22.5,40.5,0.)
(22.5,41.5,0.)
(22.5,42.5,0.)
(22.5,43.5,0.)
(22.5,44.5,0.)
(22.5,45.5,0.)
(22.5,46.5,0.)
(22.5,47.5,0.)
(22.5,48.5,0.)
(22.5,49.5,0.)
(22.5,50.5,0)

(23.5,-0.5,1.)
(23.5,0.5,1.)
(23.5,1.5,1.)
(23.5,2.5,1.)
(23.5,3.5,1.)
(23.5,4.5,1.)
(23.5,5.5,1.)
(23.5,6.5,1.)
(23.5,7.5,1.)
(23.5,8.5,1.)
(23.5,9.5,1.)
(23.5,10.5,1.)
(23.5,11.5,1.)
(23.5,12.5,1.)
(23.5,13.5,1.)
(23.5,14.5,1.)
(23.5,15.5,1.)
(23.5,16.5,1.)
(23.5,17.5,1.)
(23.5,18.5,1.)
(23.5,19.5,1.)
(23.5,20.5,1.)
(23.5,21.5,0.8)
(23.5,22.5,0.9)
(23.5,23.5,0.5)
(23.5,24.5,0.4)
(23.5,25.5,0.2)
(23.5,26.5,0.3)
(23.5,27.5,0.2)
(23.5,28.5,0.2)
(23.5,29.5,0.)
(23.5,30.5,0.)
(23.5,31.5,0.1)
(23.5,32.5,0.1)
(23.5,33.5,0.)
(23.5,34.5,0.)
(23.5,35.5,0.)
(23.5,36.5,0.)
(23.5,37.5,0.)
(23.5,38.5,0.)
(23.5,39.5,0.)
(23.5,40.5,0.)
(23.5,41.5,0.)
(23.5,42.5,0.)
(23.5,43.5,0.)
(23.5,44.5,0.)
(23.5,45.5,0.)
(23.5,46.5,0.)
(23.5,47.5,0.)
(23.5,48.5,0.)
(23.5,49.5,0.)
(23.5,50.5,0)

(24.5,-0.5,1.)
(24.5,0.5,1.)
(24.5,1.5,1.)
(24.5,2.5,1.)
(24.5,3.5,1.)
(24.5,4.5,1.)
(24.5,5.5,1.)
(24.5,6.5,1.)
(24.5,7.5,1.)
(24.5,8.5,1.)
(24.5,9.5,1.)
(24.5,10.5,1.)
(24.5,11.5,1.)
(24.5,12.5,1.)
(24.5,13.5,1.)
(24.5,14.5,1.)
(24.5,15.5,1.)
(24.5,16.5,1.)
(24.5,17.5,1.)
(24.5,18.5,1.)
(24.5,19.5,1.)
(24.5,20.5,1.)
(24.5,21.5,0.9)
(24.5,22.5,0.8)
(24.5,23.5,0.8)
(24.5,24.5,0.8)
(24.5,25.5,0.5)
(24.5,26.5,0.6)
(24.5,27.5,0.4)
(24.5,28.5,0.4)
(24.5,29.5,0.1)
(24.5,30.5,0.3)
(24.5,31.5,0.2)
(24.5,32.5,0.1)
(24.5,33.5,0.)
(24.5,34.5,0.)
(24.5,35.5,0.)
(24.5,36.5,0.)
(24.5,37.5,0.)
(24.5,38.5,0.)
(24.5,39.5,0.)
(24.5,40.5,0.)
(24.5,41.5,0.)
(24.5,42.5,0.)
(24.5,43.5,0.)
(24.5,44.5,0.)
(24.5,45.5,0.)
(24.5,46.5,0.)
(24.5,47.5,0.)
(24.5,48.5,0.)
(24.5,49.5,0.)
(24.5,50.5,0)

(25.5,-0.5,1.)
(25.5,0.5,1.)
(25.5,1.5,1.)
(25.5,2.5,1.)
(25.5,3.5,1.)
(25.5,4.5,1.)
(25.5,5.5,1.)
(25.5,6.5,1.)
(25.5,7.5,1.)
(25.5,8.5,1.)
(25.5,9.5,1.)
(25.5,10.5,1.)
(25.5,11.5,1.)
(25.5,12.5,1.)
(25.5,13.5,1.)
(25.5,14.5,1.)
(25.5,15.5,1.)
(25.5,16.5,1.)
(25.5,17.5,1.)
(25.5,18.5,0.9)
(25.5,19.5,1.)
(25.5,20.5,1.)
(25.5,21.5,1.)
(25.5,22.5,0.9)
(25.5,23.5,0.7)
(25.5,24.5,0.9)
(25.5,25.5,0.3)
(25.5,26.5,0.3)
(25.5,27.5,0.5)
(25.5,28.5,0.4)
(25.5,29.5,0.2)
(25.5,30.5,0.3)
(25.5,31.5,0.1)
(25.5,32.5,0.)
(25.5,33.5,0.1)
(25.5,34.5,0.)
(25.5,35.5,0.)
(25.5,36.5,0.)
(25.5,37.5,0.)
(25.5,38.5,0.)
(25.5,39.5,0.)
(25.5,40.5,0.)
(25.5,41.5,0.)
(25.5,42.5,0.)
(25.5,43.5,0.)
(25.5,44.5,0.)
(25.5,45.5,0.)
(25.5,46.5,0.)
(25.5,47.5,0.)
(25.5,48.5,0.)
(25.5,49.5,0.)
(25.5,50.5,0)

(26.5,-0.5,1.)
(26.5,0.5,1.)
(26.5,1.5,1.)
(26.5,2.5,1.)
(26.5,3.5,1.)
(26.5,4.5,1.)
(26.5,5.5,1.)
(26.5,6.5,1.)
(26.5,7.5,1.)
(26.5,8.5,1.)
(26.5,9.5,1.)
(26.5,10.5,1.)
(26.5,11.5,1.)
(26.5,12.5,1.)
(26.5,13.5,1.)
(26.5,14.5,1.)
(26.5,15.5,1.)
(26.5,16.5,1.)
(26.5,17.5,1.)
(26.5,18.5,1.)
(26.5,19.5,1.)
(26.5,20.5,1.)
(26.5,21.5,1.)
(26.5,22.5,1.)
(26.5,23.5,1.)
(26.5,24.5,0.9)
(26.5,25.5,0.6)
(26.5,26.5,0.5)
(26.5,27.5,0.1)
(26.5,28.5,0.4)
(26.5,29.5,0.3)
(26.5,30.5,0.2)
(26.5,31.5,0.2)
(26.5,32.5,0.1)
(26.5,33.5,0.)
(26.5,34.5,0.)
(26.5,35.5,0.1)
(26.5,36.5,0.)
(26.5,37.5,0.)
(26.5,38.5,0.)
(26.5,39.5,0.)
(26.5,40.5,0.)
(26.5,41.5,0.)
(26.5,42.5,0.)
(26.5,43.5,0.)
(26.5,44.5,0.)
(26.5,45.5,0.)
(26.5,46.5,0.)
(26.5,47.5,0.)
(26.5,48.5,0.)
(26.5,49.5,0.)
(26.5,50.5,0)

(27.5,-0.5,1.)
(27.5,0.5,1.)
(27.5,1.5,1.)
(27.5,2.5,1.)
(27.5,3.5,1.)
(27.5,4.5,1.)
(27.5,5.5,1.)
(27.5,6.5,1.)
(27.5,7.5,1.)
(27.5,8.5,1.)
(27.5,9.5,1.)
(27.5,10.5,1.)
(27.5,11.5,1.)
(27.5,12.5,1.)
(27.5,13.5,1.)
(27.5,14.5,1.)
(27.5,15.5,1.)
(27.5,16.5,1.)
(27.5,17.5,1.)
(27.5,18.5,1.)
(27.5,19.5,1.)
(27.5,20.5,1.)
(27.5,21.5,1.)
(27.5,22.5,1.)
(27.5,23.5,1.)
(27.5,24.5,0.9)
(27.5,25.5,0.7)
(27.5,26.5,0.6)
(27.5,27.5,0.2)
(27.5,28.5,0.5)
(27.5,29.5,0.2)
(27.5,30.5,0.2)
(27.5,31.5,0.3)
(27.5,32.5,0.1)
(27.5,33.5,0.1)
(27.5,34.5,0.2)
(27.5,35.5,0.)
(27.5,36.5,0.)
(27.5,37.5,0.1)
(27.5,38.5,0.)
(27.5,39.5,0.1)
(27.5,40.5,0.)
(27.5,41.5,0.)
(27.5,42.5,0.)
(27.5,43.5,0.)
(27.5,44.5,0.)
(27.5,45.5,0.)
(27.5,46.5,0.)
(27.5,47.5,0.)
(27.5,48.5,0.)
(27.5,49.5,0.)
(27.5,50.5,0)

(28.5,-0.5,1.)
(28.5,0.5,1.)
(28.5,1.5,1.)
(28.5,2.5,1.)
(28.5,3.5,1.)
(28.5,4.5,1.)
(28.5,5.5,1.)
(28.5,6.5,1.)
(28.5,7.5,1.)
(28.5,8.5,1.)
(28.5,9.5,1.)
(28.5,10.5,1.)
(28.5,11.5,1.)
(28.5,12.5,1.)
(28.5,13.5,1.)
(28.5,14.5,1.)
(28.5,15.5,1.)
(28.5,16.5,1.)
(28.5,17.5,1.)
(28.5,18.5,1.)
(28.5,19.5,1.)
(28.5,20.5,1.)
(28.5,21.5,1.)
(28.5,22.5,1.)
(28.5,23.5,1.)
(28.5,24.5,1.)
(28.5,25.5,1.)
(28.5,26.5,0.7)
(28.5,27.5,0.8)
(28.5,28.5,0.3)
(28.5,29.5,0.4)
(28.5,30.5,0.3)
(28.5,31.5,0.3)
(28.5,32.5,0.3)
(28.5,33.5,0.2)
(28.5,34.5,0.)
(28.5,35.5,0.)
(28.5,36.5,0.)
(28.5,37.5,0.1)
(28.5,38.5,0.)
(28.5,39.5,0.)
(28.5,40.5,0.)
(28.5,41.5,0.)
(28.5,42.5,0.)
(28.5,43.5,0.)
(28.5,44.5,0.)
(28.5,45.5,0.)
(28.5,46.5,0.)
(28.5,47.5,0.)
(28.5,48.5,0.)
(28.5,49.5,0.)
(28.5,50.5,0)

(29.5,-0.5,1.)
(29.5,0.5,1.)
(29.5,1.5,1.)
(29.5,2.5,1.)
(29.5,3.5,1.)
(29.5,4.5,1.)
(29.5,5.5,1.)
(29.5,6.5,1.)
(29.5,7.5,1.)
(29.5,8.5,1.)
(29.5,9.5,1.)
(29.5,10.5,1.)
(29.5,11.5,1.)
(29.5,12.5,1.)
(29.5,13.5,1.)
(29.5,14.5,1.)
(29.5,15.5,1.)
(29.5,16.5,1.)
(29.5,17.5,1.)
(29.5,18.5,1.)
(29.5,19.5,1.)
(29.5,20.5,1.)
(29.5,21.5,1.)
(29.5,22.5,1.)
(29.5,23.5,1.)
(29.5,24.5,1.)
(29.5,25.5,1.)
(29.5,26.5,0.9)
(29.5,27.5,0.9)
(29.5,28.5,0.5)
(29.5,29.5,0.4)
(29.5,30.5,0.3)
(29.5,31.5,0.1)
(29.5,32.5,0.1)
(29.5,33.5,0.1)
(29.5,34.5,0.1)
(29.5,35.5,0.)
(29.5,36.5,0.)
(29.5,37.5,0.)
(29.5,38.5,0.)
(29.5,39.5,0.)
(29.5,40.5,0.)
(29.5,41.5,0.)
(29.5,42.5,0.)
(29.5,43.5,0.)
(29.5,44.5,0.)
(29.5,45.5,0.)
(29.5,46.5,0.)
(29.5,47.5,0.)
(29.5,48.5,0.)
(29.5,49.5,0.)
(29.5,50.5,0)

(30.5,-0.5,1.)
(30.5,0.5,1.)
(30.5,1.5,1.)
(30.5,2.5,1.)
(30.5,3.5,1.)
(30.5,4.5,1.)
(30.5,5.5,1.)
(30.5,6.5,1.)
(30.5,7.5,1.)
(30.5,8.5,1.)
(30.5,9.5,1.)
(30.5,10.5,1.)
(30.5,11.5,1.)
(30.5,12.5,1.)
(30.5,13.5,1.)
(30.5,14.5,1.)
(30.5,15.5,1.)
(30.5,16.5,1.)
(30.5,17.5,1.)
(30.5,18.5,1.)
(30.5,19.5,1.)
(30.5,20.5,1.)
(30.5,21.5,1.)
(30.5,22.5,1.)
(30.5,23.5,1.)
(30.5,24.5,1.)
(30.5,25.5,0.9)
(30.5,26.5,1.)
(30.5,27.5,0.9)
(30.5,28.5,0.7)
(30.5,29.5,0.6)
(30.5,30.5,0.6)
(30.5,31.5,0.1)
(30.5,32.5,0.2)
(30.5,33.5,0.1)
(30.5,34.5,0.3)
(30.5,35.5,0.)
(30.5,36.5,0.2)
(30.5,37.5,0.)
(30.5,38.5,0.)
(30.5,39.5,0.)
(30.5,40.5,0.)
(30.5,41.5,0.)
(30.5,42.5,0.)
(30.5,43.5,0.)
(30.5,44.5,0.)
(30.5,45.5,0.)
(30.5,46.5,0.)
(30.5,47.5,0.)
(30.5,48.5,0.)
(30.5,49.5,0.)
(30.5,50.5,0)

(31.5,-0.5,1.)
(31.5,0.5,1.)
(31.5,1.5,1.)
(31.5,2.5,1.)
(31.5,3.5,1.)
(31.5,4.5,1.)
(31.5,5.5,1.)
(31.5,6.5,1.)
(31.5,7.5,1.)
(31.5,8.5,1.)
(31.5,9.5,1.)
(31.5,10.5,1.)
(31.5,11.5,1.)
(31.5,12.5,1.)
(31.5,13.5,1.)
(31.5,14.5,1.)
(31.5,15.5,1.)
(31.5,16.5,1.)
(31.5,17.5,1.)
(31.5,18.5,1.)
(31.5,19.5,1.)
(31.5,20.5,1.)
(31.5,21.5,1.)
(31.5,22.5,1.)
(31.5,23.5,1.)
(31.5,24.5,1.)
(31.5,25.5,1.)
(31.5,26.5,1.)
(31.5,27.5,0.9)
(31.5,28.5,0.6)
(31.5,29.5,0.7)
(31.5,30.5,0.4)
(31.5,31.5,0.3)
(31.5,32.5,0.3)
(31.5,33.5,0.4)
(31.5,34.5,0.2)
(31.5,35.5,0.1)
(31.5,36.5,0.1)
(31.5,37.5,0.)
(31.5,38.5,0.1)
(31.5,39.5,0.)
(31.5,40.5,0.)
(31.5,41.5,0.)
(31.5,42.5,0.)
(31.5,43.5,0.)
(31.5,44.5,0.)
(31.5,45.5,0.)
(31.5,46.5,0.)
(31.5,47.5,0.)
(31.5,48.5,0.)
(31.5,49.5,0.)
(31.5,50.5,0)

(32.5,-0.5,1.)
(32.5,0.5,1.)
(32.5,1.5,1.)
(32.5,2.5,1.)
(32.5,3.5,1.)
(32.5,4.5,1.)
(32.5,5.5,1.)
(32.5,6.5,1.)
(32.5,7.5,1.)
(32.5,8.5,1.)
(32.5,9.5,1.)
(32.5,10.5,1.)
(32.5,11.5,1.)
(32.5,12.5,1.)
(32.5,13.5,1.)
(32.5,14.5,1.)
(32.5,15.5,1.)
(32.5,16.5,1.)
(32.5,17.5,1.)
(32.5,18.5,1.)
(32.5,19.5,1.)
(32.5,20.5,1.)
(32.5,21.5,1.)
(32.5,22.5,1.)
(32.5,23.5,1.)
(32.5,24.5,1.)
(32.5,25.5,1.)
(32.5,26.5,1.)
(32.5,27.5,1.)
(32.5,28.5,0.9)
(32.5,29.5,0.7)
(32.5,30.5,0.6)
(32.5,31.5,0.5)
(32.5,32.5,0.5)
(32.5,33.5,0.2)
(32.5,34.5,0.)
(32.5,35.5,0.1)
(32.5,36.5,0.)
(32.5,37.5,0.1)
(32.5,38.5,0.)
(32.5,39.5,0.)
(32.5,40.5,0.)
(32.5,41.5,0.)
(32.5,42.5,0.)
(32.5,43.5,0.)
(32.5,44.5,0.)
(32.5,45.5,0.)
(32.5,46.5,0.)
(32.5,47.5,0.)
(32.5,48.5,0.)
(32.5,49.5,0.)
(32.5,50.5,0)

(33.5,-0.5,1.)
(33.5,0.5,1.)
(33.5,1.5,1.)
(33.5,2.5,1.)
(33.5,3.5,1.)
(33.5,4.5,1.)
(33.5,5.5,1.)
(33.5,6.5,1.)
(33.5,7.5,1.)
(33.5,8.5,1.)
(33.5,9.5,1.)
(33.5,10.5,1.)
(33.5,11.5,1.)
(33.5,12.5,1.)
(33.5,13.5,1.)
(33.5,14.5,1.)
(33.5,15.5,1.)
(33.5,16.5,1.)
(33.5,17.5,1.)
(33.5,18.5,1.)
(33.5,19.5,1.)
(33.5,20.5,1.)
(33.5,21.5,1.)
(33.5,22.5,1.)
(33.5,23.5,1.)
(33.5,24.5,1.)
(33.5,25.5,1.)
(33.5,26.5,1.)
(33.5,27.5,1.)
(33.5,28.5,1.)
(33.5,29.5,1.)
(33.5,30.5,0.9)
(33.5,31.5,0.5)
(33.5,32.5,0.5)
(33.5,33.5,0.5)
(33.5,34.5,0.3)
(33.5,35.5,0.2)
(33.5,36.5,0.2)
(33.5,37.5,0.)
(33.5,38.5,0.)
(33.5,39.5,0.)
(33.5,40.5,0.)
(33.5,41.5,0.)
(33.5,42.5,0.1)
(33.5,43.5,0.1)
(33.5,44.5,0.)
(33.5,45.5,0.)
(33.5,46.5,0.)
(33.5,47.5,0.)
(33.5,48.5,0.)
(33.5,49.5,0.)
(33.5,50.5,0)

(34.5,-0.5,1.)
(34.5,0.5,1.)
(34.5,1.5,1.)
(34.5,2.5,1.)
(34.5,3.5,1.)
(34.5,4.5,1.)
(34.5,5.5,1.)
(34.5,6.5,1.)
(34.5,7.5,1.)
(34.5,8.5,1.)
(34.5,9.5,1.)
(34.5,10.5,1.)
(34.5,11.5,1.)
(34.5,12.5,1.)
(34.5,13.5,1.)
(34.5,14.5,1.)
(34.5,15.5,1.)
(34.5,16.5,1.)
(34.5,17.5,1.)
(34.5,18.5,1.)
(34.5,19.5,1.)
(34.5,20.5,1.)
(34.5,21.5,1.)
(34.5,22.5,1.)
(34.5,23.5,1.)
(34.5,24.5,1.)
(34.5,25.5,1.)
(34.5,26.5,1.)
(34.5,27.5,1.)
(34.5,28.5,1.)
(34.5,29.5,1.)
(34.5,30.5,0.9)
(34.5,31.5,0.8)
(34.5,32.5,0.5)
(34.5,33.5,0.7)
(34.5,34.5,0.3)
(34.5,35.5,0.3)
(34.5,36.5,0.2)
(34.5,37.5,0.3)
(34.5,38.5,0.1)
(34.5,39.5,0.)
(34.5,40.5,0.)
(34.5,41.5,0.1)
(34.5,42.5,0.)
(34.5,43.5,0.)
(34.5,44.5,0.)
(34.5,45.5,0.)
(34.5,46.5,0.)
(34.5,47.5,0.)
(34.5,48.5,0.)
(34.5,49.5,0.)
(34.5,50.5,0)

(35.5,-0.5,1.)
(35.5,0.5,1.)
(35.5,1.5,1.)
(35.5,2.5,1.)
(35.5,3.5,1.)
(35.5,4.5,1.)
(35.5,5.5,1.)
(35.5,6.5,1.)
(35.5,7.5,1.)
(35.5,8.5,1.)
(35.5,9.5,1.)
(35.5,10.5,1.)
(35.5,11.5,1.)
(35.5,12.5,1.)
(35.5,13.5,1.)
(35.5,14.5,1.)
(35.5,15.5,1.)
(35.5,16.5,1.)
(35.5,17.5,1.)
(35.5,18.5,1.)
(35.5,19.5,1.)
(35.5,20.5,1.)
(35.5,21.5,1.)
(35.5,22.5,1.)
(35.5,23.5,1.)
(35.5,24.5,1.)
(35.5,25.5,1.)
(35.5,26.5,1.)
(35.5,27.5,1.)
(35.5,28.5,1.)
(35.5,29.5,1.)
(35.5,30.5,1.)
(35.5,31.5,0.9)
(35.5,32.5,0.8)
(35.5,33.5,0.8)
(35.5,34.5,0.6)
(35.5,35.5,0.2)
(35.5,36.5,0.3)
(35.5,37.5,0.)
(35.5,38.5,0.2)
(35.5,39.5,0.1)
(35.5,40.5,0.1)
(35.5,41.5,0.)
(35.5,42.5,0.)
(35.5,43.5,0.)
(35.5,44.5,0.)
(35.5,45.5,0.)
(35.5,46.5,0.)
(35.5,47.5,0.)
(35.5,48.5,0.)
(35.5,49.5,0.)
(35.5,50.5,0)

(36.5,-0.5,1.)
(36.5,0.5,1.)
(36.5,1.5,1.)
(36.5,2.5,1.)
(36.5,3.5,1.)
(36.5,4.5,1.)
(36.5,5.5,1.)
(36.5,6.5,1.)
(36.5,7.5,1.)
(36.5,8.5,1.)
(36.5,9.5,1.)
(36.5,10.5,1.)
(36.5,11.5,1.)
(36.5,12.5,1.)
(36.5,13.5,1.)
(36.5,14.5,1.)
(36.5,15.5,1.)
(36.5,16.5,1.)
(36.5,17.5,1.)
(36.5,18.5,1.)
(36.5,19.5,1.)
(36.5,20.5,1.)
(36.5,21.5,1.)
(36.5,22.5,1.)
(36.5,23.5,1.)
(36.5,24.5,1.)
(36.5,25.5,1.)
(36.5,26.5,1.)
(36.5,27.5,1.)
(36.5,28.5,1.)
(36.5,29.5,1.)
(36.5,30.5,1.)
(36.5,31.5,0.9)
(36.5,32.5,1.)
(36.5,33.5,0.9)
(36.5,34.5,0.7)
(36.5,35.5,0.3)
(36.5,36.5,0.5)
(36.5,37.5,0.2)
(36.5,38.5,0.3)
(36.5,39.5,0.)
(36.5,40.5,0.)
(36.5,41.5,0.)
(36.5,42.5,0.)
(36.5,43.5,0.)
(36.5,44.5,0.)
(36.5,45.5,0.)
(36.5,46.5,0.)
(36.5,47.5,0.)
(36.5,48.5,0.)
(36.5,49.5,0.)
(36.5,50.5,0)

(37.5,-0.5,1.)
(37.5,0.5,1.)
(37.5,1.5,1.)
(37.5,2.5,1.)
(37.5,3.5,1.)
(37.5,4.5,1.)
(37.5,5.5,1.)
(37.5,6.5,1.)
(37.5,7.5,1.)
(37.5,8.5,1.)
(37.5,9.5,1.)
(37.5,10.5,1.)
(37.5,11.5,1.)
(37.5,12.5,1.)
(37.5,13.5,1.)
(37.5,14.5,1.)
(37.5,15.5,1.)
(37.5,16.5,1.)
(37.5,17.5,1.)
(37.5,18.5,1.)
(37.5,19.5,1.)
(37.5,20.5,1.)
(37.5,21.5,1.)
(37.5,22.5,1.)
(37.5,23.5,1.)
(37.5,24.5,1.)
(37.5,25.5,1.)
(37.5,26.5,1.)
(37.5,27.5,1.)
(37.5,28.5,1.)
(37.5,29.5,1.)
(37.5,30.5,1.)
(37.5,31.5,1.)
(37.5,32.5,1.)
(37.5,33.5,1.)
(37.5,34.5,0.8)
(37.5,35.5,0.8)
(37.5,36.5,0.8)
(37.5,37.5,0.3)
(37.5,38.5,0.1)
(37.5,39.5,0.1)
(37.5,40.5,0.3)
(37.5,41.5,0.1)
(37.5,42.5,0.)
(37.5,43.5,0.)
(37.5,44.5,0.)
(37.5,45.5,0.)
(37.5,46.5,0.)
(37.5,47.5,0.)
(37.5,48.5,0.)
(37.5,49.5,0.)
(37.5,50.5,0)

(38.5,-0.5,1.)
(38.5,0.5,1.)
(38.5,1.5,1.)
(38.5,2.5,1.)
(38.5,3.5,1.)
(38.5,4.5,1.)
(38.5,5.5,1.)
(38.5,6.5,1.)
(38.5,7.5,1.)
(38.5,8.5,1.)
(38.5,9.5,1.)
(38.5,10.5,1.)
(38.5,11.5,1.)
(38.5,12.5,1.)
(38.5,13.5,1.)
(38.5,14.5,1.)
(38.5,15.5,1.)
(38.5,16.5,1.)
(38.5,17.5,1.)
(38.5,18.5,1.)
(38.5,19.5,1.)
(38.5,20.5,1.)
(38.5,21.5,1.)
(38.5,22.5,1.)
(38.5,23.5,1.)
(38.5,24.5,1.)
(38.5,25.5,1.)
(38.5,26.5,1.)
(38.5,27.5,1.)
(38.5,28.5,1.)
(38.5,29.5,1.)
(38.5,30.5,1.)
(38.5,31.5,1.)
(38.5,32.5,1.)
(38.5,33.5,1.)
(38.5,34.5,1.)
(38.5,35.5,0.7)
(38.5,36.5,0.9)
(38.5,37.5,0.1)
(38.5,38.5,0.3)
(38.5,39.5,0.2)
(38.5,40.5,0.2)
(38.5,41.5,0.)
(38.5,42.5,0.)
(38.5,43.5,0.)
(38.5,44.5,0.)
(38.5,45.5,0.)
(38.5,46.5,0.)
(38.5,47.5,0.1)
(38.5,48.5,0.)
(38.5,49.5,0.)
(38.5,50.5,0)

(39.5,-0.5,1.)
(39.5,0.5,1.)
(39.5,1.5,1.)
(39.5,2.5,1.)
(39.5,3.5,1.)
(39.5,4.5,1.)
(39.5,5.5,1.)
(39.5,6.5,1.)
(39.5,7.5,1.)
(39.5,8.5,1.)
(39.5,9.5,1.)
(39.5,10.5,1.)
(39.5,11.5,1.)
(39.5,12.5,1.)
(39.5,13.5,1.)
(39.5,14.5,1.)
(39.5,15.5,1.)
(39.5,16.5,1.)
(39.5,17.5,1.)
(39.5,18.5,1.)
(39.5,19.5,1.)
(39.5,20.5,1.)
(39.5,21.5,1.)
(39.5,22.5,1.)
(39.5,23.5,1.)
(39.5,24.5,1.)
(39.5,25.5,1.)
(39.5,26.5,1.)
(39.5,27.5,1.)
(39.5,28.5,1.)
(39.5,29.5,1.)
(39.5,30.5,1.)
(39.5,31.5,1.)
(39.5,32.5,1.)
(39.5,33.5,1.)
(39.5,34.5,0.9)
(39.5,35.5,1.)
(39.5,36.5,0.9)
(39.5,37.5,0.6)
(39.5,38.5,0.8)
(39.5,39.5,0.4)
(39.5,40.5,0.1)
(39.5,41.5,0.2)
(39.5,42.5,0.2)
(39.5,43.5,0.1)
(39.5,44.5,0.)
(39.5,45.5,0.)
(39.5,46.5,0.1)
(39.5,47.5,0.)
(39.5,48.5,0.)
(39.5,49.5,0.)
(39.5,50.5,0)

(40.5,-0.5,1.)
(40.5,0.5,1.)
(40.5,1.5,1.)
(40.5,2.5,1.)
(40.5,3.5,1.)
(40.5,4.5,1.)
(40.5,5.5,1.)
(40.5,6.5,1.)
(40.5,7.5,1.)
(40.5,8.5,1.)
(40.5,9.5,1.)
(40.5,10.5,1.)
(40.5,11.5,1.)
(40.5,12.5,1.)
(40.5,13.5,1.)
(40.5,14.5,1.)
(40.5,15.5,1.)
(40.5,16.5,1.)
(40.5,17.5,1.)
(40.5,18.5,1.)
(40.5,19.5,1.)
(40.5,20.5,1.)
(40.5,21.5,1.)
(40.5,22.5,1.)
(40.5,23.5,1.)
(40.5,24.5,1.)
(40.5,25.5,1.)
(40.5,26.5,1.)
(40.5,27.5,1.)
(40.5,28.5,1.)
(40.5,29.5,1.)
(40.5,30.5,1.)
(40.5,31.5,1.)
(40.5,32.5,1.)
(40.5,33.5,1.)
(40.5,34.5,1.)
(40.5,35.5,1.)
(40.5,36.5,1.)
(40.5,37.5,0.8)
(40.5,38.5,0.8)
(40.5,39.5,0.9)
(40.5,40.5,0.4)
(40.5,41.5,0.)
(40.5,42.5,0.)
(40.5,43.5,0.1)
(40.5,44.5,0.)
(40.5,45.5,0.)
(40.5,46.5,0.)
(40.5,47.5,0.)
(40.5,48.5,0.)
(40.5,49.5,0.)
(40.5,50.5,0)

(41.5,-0.5,1.)
(41.5,0.5,1.)
(41.5,1.5,1.)
(41.5,2.5,1.)
(41.5,3.5,1.)
(41.5,4.5,1.)
(41.5,5.5,1.)
(41.5,6.5,1.)
(41.5,7.5,1.)
(41.5,8.5,1.)
(41.5,9.5,1.)
(41.5,10.5,1.)
(41.5,11.5,1.)
(41.5,12.5,1.)
(41.5,13.5,1.)
(41.5,14.5,1.)
(41.5,15.5,1.)
(41.5,16.5,1.)
(41.5,17.5,1.)
(41.5,18.5,1.)
(41.5,19.5,1.)
(41.5,20.5,1.)
(41.5,21.5,1.)
(41.5,22.5,1.)
(41.5,23.5,1.)
(41.5,24.5,1.)
(41.5,25.5,1.)
(41.5,26.5,1.)
(41.5,27.5,1.)
(41.5,28.5,1.)
(41.5,29.5,1.)
(41.5,30.5,1.)
(41.5,31.5,1.)
(41.5,32.5,1.)
(41.5,33.5,1.)
(41.5,34.5,1.)
(41.5,35.5,1.)
(41.5,36.5,1.)
(41.5,37.5,1.)
(41.5,38.5,0.9)
(41.5,39.5,0.6)
(41.5,40.5,0.2)
(41.5,41.5,0.3)
(41.5,42.5,0.)
(41.5,43.5,0.1)
(41.5,44.5,0.)
(41.5,45.5,0.1)
(41.5,46.5,0.)
(41.5,47.5,0.)
(41.5,48.5,0.)
(41.5,49.5,0.1)
(41.5,50.5,0)

(42.5,-0.5,1.)
(42.5,0.5,1.)
(42.5,1.5,1.)
(42.5,2.5,1.)
(42.5,3.5,1.)
(42.5,4.5,1.)
(42.5,5.5,1.)
(42.5,6.5,1.)
(42.5,7.5,1.)
(42.5,8.5,1.)
(42.5,9.5,1.)
(42.5,10.5,1.)
(42.5,11.5,1.)
(42.5,12.5,1.)
(42.5,13.5,1.)
(42.5,14.5,1.)
(42.5,15.5,1.)
(42.5,16.5,1.)
(42.5,17.5,1.)
(42.5,18.5,1.)
(42.5,19.5,1.)
(42.5,20.5,1.)
(42.5,21.5,1.)
(42.5,22.5,1.)
(42.5,23.5,1.)
(42.5,24.5,1.)
(42.5,25.5,1.)
(42.5,26.5,1.)
(42.5,27.5,1.)
(42.5,28.5,1.)
(42.5,29.5,1.)
(42.5,30.5,1.)
(42.5,31.5,1.)
(42.5,32.5,1.)
(42.5,33.5,1.)
(42.5,34.5,1.)
(42.5,35.5,1.)
(42.5,36.5,1.)
(42.5,37.5,1.)
(42.5,38.5,1.)
(42.5,39.5,1.)
(42.5,40.5,0.7)
(42.5,41.5,0.4)
(42.5,42.5,0.6)
(42.5,43.5,0.1)
(42.5,44.5,0.2)
(42.5,45.5,0.)
(42.5,46.5,0.1)
(42.5,47.5,0.)
(42.5,48.5,0.)
(42.5,49.5,0.)
(42.5,50.5,0)

(43.5,-0.5,1.)
(43.5,0.5,1.)
(43.5,1.5,1.)
(43.5,2.5,1.)
(43.5,3.5,1.)
(43.5,4.5,1.)
(43.5,5.5,1.)
(43.5,6.5,1.)
(43.5,7.5,1.)
(43.5,8.5,1.)
(43.5,9.5,1.)
(43.5,10.5,1.)
(43.5,11.5,1.)
(43.5,12.5,1.)
(43.5,13.5,1.)
(43.5,14.5,1.)
(43.5,15.5,1.)
(43.5,16.5,1.)
(43.5,17.5,1.)
(43.5,18.5,1.)
(43.5,19.5,1.)
(43.5,20.5,1.)
(43.5,21.5,1.)
(43.5,22.5,1.)
(43.5,23.5,1.)
(43.5,24.5,1.)
(43.5,25.5,1.)
(43.5,26.5,1.)
(43.5,27.5,1.)
(43.5,28.5,1.)
(43.5,29.5,1.)
(43.5,30.5,1.)
(43.5,31.5,1.)
(43.5,32.5,1.)
(43.5,33.5,1.)
(43.5,34.5,1.)
(43.5,35.5,1.)
(43.5,36.5,1.)
(43.5,37.5,1.)
(43.5,38.5,1.)
(43.5,39.5,0.9)
(43.5,40.5,0.9)
(43.5,41.5,0.9)
(43.5,42.5,0.5)
(43.5,43.5,0.3)
(43.5,44.5,0.2)
(43.5,45.5,0.1)
(43.5,46.5,0.1)
(43.5,47.5,0.1)
(43.5,48.5,0.)
(43.5,49.5,0.)
(43.5,50.5,0)

(44.5,-0.5,1.)
(44.5,0.5,1.)
(44.5,1.5,1.)
(44.5,2.5,1.)
(44.5,3.5,1.)
(44.5,4.5,1.)
(44.5,5.5,1.)
(44.5,6.5,1.)
(44.5,7.5,1.)
(44.5,8.5,1.)
(44.5,9.5,1.)
(44.5,10.5,1.)
(44.5,11.5,1.)
(44.5,12.5,1.)
(44.5,13.5,1.)
(44.5,14.5,1.)
(44.5,15.5,1.)
(44.5,16.5,1.)
(44.5,17.5,1.)
(44.5,18.5,1.)
(44.5,19.5,1.)
(44.5,20.5,1.)
(44.5,21.5,1.)
(44.5,22.5,1.)
(44.5,23.5,1.)
(44.5,24.5,1.)
(44.5,25.5,1.)
(44.5,26.5,1.)
(44.5,27.5,1.)
(44.5,28.5,1.)
(44.5,29.5,1.)
(44.5,30.5,1.)
(44.5,31.5,1.)
(44.5,32.5,1.)
(44.5,33.5,1.)
(44.5,34.5,1.)
(44.5,35.5,1.)
(44.5,36.5,1.)
(44.5,37.5,1.)
(44.5,38.5,1.)
(44.5,39.5,1.)
(44.5,40.5,1.)
(44.5,41.5,1.)
(44.5,42.5,0.9)
(44.5,43.5,0.7)
(44.5,44.5,0.4)
(44.5,45.5,0.)
(44.5,46.5,0.1)
(44.5,47.5,0.)
(44.5,48.5,0.)
(44.5,49.5,0.1)
(44.5,50.5,0)

(45.5,-0.5,0)
(45.5,0.5,0)
(45.5,1.5,0)
(45.5,2.5,0)
(45.5,3.5,0)
(45.5,4.5,0)
(45.5,5.5,0)
(45.5,6.5,0)
(45.5,7.5,0)
(45.5,8.5,0)
(45.5,9.5,0)
(45.5,10.5,0)
(45.5,11.5,0)
(45.5,12.5,0)
(45.5,13.5,0)
(45.5,14.5,0)
(45.5,15.5,0)
(45.5,16.5,0)
(45.5,17.5,0)
(45.5,18.5,0)
(45.5,19.5,0)
(45.5,20.5,0)
(45.5,21.5,0)
(45.5,22.5,0)
(45.5,23.5,0)
(45.5,24.5,0)
(45.5,25.5,0)
(45.5,26.5,0)
(45.5,27.5,0)
(45.5,28.5,0)
(45.5,29.5,0)
(45.5,30.5,0)
(45.5,31.5,0)
(45.5,32.5,0)
(45.5,33.5,0)
(45.5,34.5,0)
(45.5,35.5,0)
(45.5,36.5,0)
(45.5,37.5,0)
(45.5,38.5,0)
(45.5,39.5,0)
(45.5,40.5,0)
(45.5,41.5,0)
(45.5,42.5,0)
(45.5,43.5,0)
(45.5,44.5,0)
(45.5,45.5,0)
(45.5,46.5,0)
(45.5,47.5,0)
(45.5,48.5,0)
(45.5,49.5,0)
(45.5,50.5,0)
};

%% file: paper-soft-locality.bbl
\begin{thebibliography}{27}%
\makeatletter
\providecommand \@ifxundefined [1]{%
 \@ifx{#1\undefined}
}%
\providecommand \@ifnum [1]{%
 \ifnum #1\expandafter \@firstoftwo
 \else \expandafter \@secondoftwo
 \fi
}%
\providecommand \@ifx [1]{%
 \ifx #1\expandafter \@firstoftwo
 \else \expandafter \@secondoftwo
 \fi
}%
\providecommand \natexlab [1]{#1}%
\providecommand \enquote  [1]{``#1''}%
\providecommand \bibnamefont  [1]{#1}%
\providecommand \bibfnamefont [1]{#1}%
\providecommand \citenamefont [1]{#1}%
\providecommand \href@noop [0]{\@secondoftwo}%
\providecommand \href [0]{\begingroup \@sanitize@url \@href}%
\providecommand \@href[1]{\@@startlink{#1}\@@href}%
\providecommand \@@href[1]{\endgroup#1\@@endlink}%
\providecommand \@sanitize@url [0]{\catcode `\\12\catcode `\$12\catcode
  `\&12\catcode `\#12\catcode `\^12\catcode `\_12\catcode `\%12\relax}%
\providecommand \@@startlink[1]{}%
\providecommand \@@endlink[0]{}%
\providecommand \url  [0]{\begingroup\@sanitize@url \@url }%
\providecommand \@url [1]{\endgroup\@href {#1}{\urlprefix }}%
\providecommand \urlprefix  [0]{URL }%
\providecommand \Eprint [0]{\href }%
\providecommand \doibase [0]{http://dx.doi.org/}%
\providecommand \selectlanguage [0]{\@gobble}%
\providecommand \bibinfo  [0]{\@secondoftwo}%
\providecommand \bibfield  [0]{\@secondoftwo}%
\providecommand \translation [1]{[#1]}%
\providecommand \BibitemOpen [0]{}%
\providecommand \bibitemStop [0]{}%
\providecommand \bibitemNoStop [0]{.\EOS\space}%
\providecommand \EOS [0]{\spacefactor3000\relax}%
\providecommand \BibitemShut  [1]{\csname bibitem#1\endcsname}%
\let\auto@bib@innerbib\@empty
\bibitem [{\citenamefont {Schaeffer}(2007)}]{Schaeffer2007}%
  \BibitemOpen
  \bibfield  {author} {\bibinfo {author} {\bibfnamefont {S.~E.}\ \bibnamefont
  {Schaeffer}},\ }\href@noop {} {\bibfield  {journal} {\bibinfo  {journal}
  {Computer Science Review}\ }\textbf {\bibinfo {volume} {1}},\ \bibinfo
  {pages} {27} (\bibinfo {year} {2007})}\BibitemShut {NoStop}%
\bibitem [{\citenamefont {Fortunato}(2010)}]{Fortunato2010}%
  \BibitemOpen
  \bibfield  {author} {\bibinfo {author} {\bibfnamefont {S.}~\bibnamefont
  {Fortunato}},\ }\href@noop {} {\bibfield  {journal} {\bibinfo  {journal}
  {Physics Reports}\ }\textbf {\bibinfo {volume} {486}},\ \bibinfo {pages} {75}
  (\bibinfo {year} {2010})}\BibitemShut {NoStop}%
\bibitem [{\citenamefont {Ackerman}\ and\ \citenamefont
  {Ben-David}(2008)}]{AckermanBen-David2008axioms}%
  \BibitemOpen
  \bibfield  {author} {\bibinfo {author} {\bibfnamefont {M.}~\bibnamefont
  {Ackerman}}\ and\ \bibinfo {author} {\bibfnamefont {S.}~\bibnamefont
  {Ben-David}},\ }in\ \href@noop {} {\emph {\bibinfo {booktitle} {NIPS}}},\
  \bibinfo {editor} {edited by\ \bibinfo {editor} {\bibfnamefont
  {D.}~\bibnamefont {Koller}}, \bibinfo {editor} {\bibfnamefont
  {D.}~\bibnamefont {Schuurmans}}, \bibinfo {editor} {\bibfnamefont
  {Y.}~\bibnamefont {Bengio}}, \ and\ \bibinfo {editor} {\bibfnamefont
  {L.}~\bibnamefont {Bottou}}}\ (\bibinfo  {publisher} {Curran Associates,
  Inc.},\ \bibinfo {year} {2008})\ pp.\ \bibinfo {pages} {121--128}\BibitemShut
  {NoStop}%
\bibitem [{\citenamefont {Traag}\ \emph {et~al.}(2011)\citenamefont {Traag},
  \citenamefont {Van~Dooren},\ and\ \citenamefont
  {Nesterov}}]{Traag2011ResolutionLimitScope}%
  \BibitemOpen
  \bibfield  {author} {\bibinfo {author} {\bibfnamefont {V.~A.}\ \bibnamefont
  {Traag}}, \bibinfo {author} {\bibfnamefont {P.}~\bibnamefont {Van~Dooren}}, \
  and\ \bibinfo {author} {\bibfnamefont {Y.~E.}\ \bibnamefont {Nesterov}},\
  }\href@noop {} {\bibfield  {journal} {\bibinfo  {journal} {Phys. Rev. E}\
  }\textbf {\bibinfo {volume} {84}},\ \bibinfo {pages} {016114} (\bibinfo
  {year} {2011})}\BibitemShut {NoStop}%
\bibitem [{\citenamefont {Fortunato}\ and\ \citenamefont
  {Barth\'elemy}(2007)}]{Fortunato2007ResolutionLimit}%
  \BibitemOpen
  \bibfield  {author} {\bibinfo {author} {\bibfnamefont {S.}~\bibnamefont
  {Fortunato}}\ and\ \bibinfo {author} {\bibfnamefont {M.}~\bibnamefont
  {Barth\'elemy}},\ }\href {\doibase 10.1073/pnas.0605965104} {\bibfield
  {journal} {\bibinfo  {journal} {Proc. Natl. Acad. Sci. USA}\ }\textbf
  {\bibinfo {volume} {104}},\ \bibinfo {pages} {36} (\bibinfo {year}
  {2007})}\BibitemShut {NoStop}%
\bibitem [{\citenamefont {Newman}\ and\ \citenamefont
  {Girvan}(2004)}]{NewmanGirvan2004}%
  \BibitemOpen
  \bibfield  {author} {\bibinfo {author} {\bibfnamefont {M.~E.~J.}\
  \bibnamefont {Newman}}\ and\ \bibinfo {author} {\bibfnamefont
  {M.}~\bibnamefont {Girvan}},\ }\href {\doibase 10.1103/PhysRevE.69.026113}
  {\bibfield  {journal} {\bibinfo  {journal} {Phys. Rev. E}\ }\textbf {\bibinfo
  {volume} {69}},\ \bibinfo {pages} {026113} (\bibinfo {year}
  {2004})}\BibitemShut {NoStop}%
\bibitem [{\citenamefont {Paatero}\ and\ \citenamefont
  {Tapper}(1994)}]{Paatero1994}%
  \BibitemOpen
  \bibfield  {author} {\bibinfo {author} {\bibfnamefont {P.}~\bibnamefont
  {Paatero}}\ and\ \bibinfo {author} {\bibfnamefont {U.}~\bibnamefont
  {Tapper}},\ }\href@noop {} {\bibfield  {journal} {\bibinfo  {journal}
  {Environmetrics}\ }\textbf {\bibinfo {volume} {5}},\ \bibinfo {pages} {111}
  (\bibinfo {year} {1994})}\BibitemShut {NoStop}%
\bibitem [{\citenamefont {Lee}\ and\ \citenamefont {Seung}(1999)}]{Lee1999}%
  \BibitemOpen
  \bibfield  {author} {\bibinfo {author} {\bibfnamefont {D.~D.}\ \bibnamefont
  {Lee}}\ and\ \bibinfo {author} {\bibfnamefont {H.~S.}\ \bibnamefont
  {Seung}},\ }\href@noop {} {\bibfield  {journal} {\bibinfo  {journal}
  {Nature}\ }\textbf {\bibinfo {volume} {401}},\ \bibinfo {pages} {788}
  (\bibinfo {year} {1999})}\BibitemShut {NoStop}%
\bibitem [{\citenamefont {Wang}\ \emph {et~al.}(2011)\citenamefont {Wang},
  \citenamefont {Li}, \citenamefont {Wang}, \citenamefont {Zhu},\ and\
  \citenamefont {Ding}}]{Wang2011}%
  \BibitemOpen
  \bibfield  {author} {\bibinfo {author} {\bibfnamefont {F.}~\bibnamefont
  {Wang}}, \bibinfo {author} {\bibfnamefont {T.}~\bibnamefont {Li}}, \bibinfo
  {author} {\bibfnamefont {X.}~\bibnamefont {Wang}}, \bibinfo {author}
  {\bibfnamefont {S.}~\bibnamefont {Zhu}}, \ and\ \bibinfo {author}
  {\bibfnamefont {C.}~\bibnamefont {Ding}},\ }\href@noop {} {\bibfield
  {journal} {\bibinfo  {journal} {Data Min. Knowl. Discov.}\ }\textbf {\bibinfo
  {volume} {22}},\ \bibinfo {pages} {493} (\bibinfo {year} {2011})}\BibitemShut
  {NoStop}%
\bibitem [{\citenamefont {Li}\ and\ \citenamefont
  {Ding}(2013)}]{LiDing2013nmfsurvey}%
  \BibitemOpen
  \bibfield  {author} {\bibinfo {author} {\bibfnamefont {T.}~\bibnamefont
  {Li}}\ and\ \bibinfo {author} {\bibfnamefont {C.~H.~Q.}\ \bibnamefont
  {Ding}},\ }in\ \href@noop {} {\emph {\bibinfo {booktitle} {Data Clustering:
  Algorithms and Applications}}}\ (\bibinfo  {publisher} {Chapman and
  Hall/CRC},\ \bibinfo {year} {2013})\ pp.\ \bibinfo {pages}
  {149--176}\BibitemShut {NoStop}%
\bibitem [{\citenamefont {Psorakis}\ \emph {et~al.}(2011)\citenamefont
  {Psorakis}, \citenamefont {Roberts}, \citenamefont {Ebden},\ and\
  \citenamefont {Sheldon}}]{Psorakis2011NMF}%
  \BibitemOpen
  \bibfield  {author} {\bibinfo {author} {\bibfnamefont {I.}~\bibnamefont
  {Psorakis}}, \bibinfo {author} {\bibfnamefont {S.}~\bibnamefont {Roberts}},
  \bibinfo {author} {\bibfnamefont {M.}~\bibnamefont {Ebden}}, \ and\ \bibinfo
  {author} {\bibfnamefont {B.}~\bibnamefont {Sheldon}},\ }\href@noop {}
  {\bibfield  {journal} {\bibinfo  {journal} {Phys. Rev. E}\ }\textbf {\bibinfo
  {volume} {83}},\ \bibinfo {pages} {066114+} (\bibinfo {year}
  {2011})}\BibitemShut {NoStop}%
\bibitem [{\citenamefont {Reichardt}\ and\ \citenamefont
  {Bornholdt}(2004)}]{Reichardt2004}%
  \BibitemOpen
  \bibfield  {author} {\bibinfo {author} {\bibfnamefont {J.}~\bibnamefont
  {Reichardt}}\ and\ \bibinfo {author} {\bibfnamefont {S.}~\bibnamefont
  {Bornholdt}},\ }\href {\doibase 10.1103/PhysRevLett.93.218701} {\bibfield
  {journal} {\bibinfo  {journal} {Phys. Rev. Lett.}\ }\textbf {\bibinfo
  {volume} {93}},\ \bibinfo {pages} {218701} (\bibinfo {year}
  {2004})}\BibitemShut {NoStop}%
\bibitem [{\citenamefont {van Laarhoven}\ and\ \citenamefont
  {Marchiori}(2013)}]{vanLaarhoven2013lso}%
  \BibitemOpen
  \bibfield  {author} {\bibinfo {author} {\bibfnamefont {T.}~\bibnamefont {van
  Laarhoven}}\ and\ \bibinfo {author} {\bibfnamefont {E.}~\bibnamefont
  {Marchiori}},\ }\href@noop {} {\bibfield  {journal} {\bibinfo  {journal}
  {Phys. Rev. E}\ }\textbf {\bibinfo {volume} {87}},\ \bibinfo {pages} {012812}
  (\bibinfo {year} {2013})}\BibitemShut {NoStop}%
\bibitem [{\citenamefont {Blondel}\ \emph {et~al.}(2008)\citenamefont
  {Blondel}, \citenamefont {Guillaume}, \citenamefont {Lambiotte},\ and\
  \citenamefont {Lefebvre}}]{Blondel2008}%
  \BibitemOpen
  \bibfield  {author} {\bibinfo {author} {\bibfnamefont {V.~D.}\ \bibnamefont
  {Blondel}}, \bibinfo {author} {\bibfnamefont {J.-L.}\ \bibnamefont
  {Guillaume}}, \bibinfo {author} {\bibfnamefont {R.}~\bibnamefont
  {Lambiotte}}, \ and\ \bibinfo {author} {\bibfnamefont {E.}~\bibnamefont
  {Lefebvre}},\ }\href {\doibase 10.1088/1742-5468/2008/10/P10008} {\bibfield
  {journal} {\bibinfo  {journal} {J. Stat. Mech. Theory Exp.}\ }\textbf
  {\bibinfo {volume} {2008}},\ \bibinfo {pages} {P10008} (\bibinfo {year}
  {2008})},\ \Eprint {http://arxiv.org/abs/0803.0476} {arXiv:0803.0476}
  \BibitemShut {NoStop}%
\bibitem [{\citenamefont {Ackerman}\ \emph {et~al.}(2010)\citenamefont
  {Ackerman}, \citenamefont {Ben-David},\ and\ \citenamefont
  {Loker}}]{AckermanBenDavidLokerCOLT2010}%
  \BibitemOpen
  \bibfield  {author} {\bibinfo {author} {\bibfnamefont {M.}~\bibnamefont
  {Ackerman}}, \bibinfo {author} {\bibfnamefont {S.}~\bibnamefont {Ben-David}},
  \ and\ \bibinfo {author} {\bibfnamefont {D.}~\bibnamefont {Loker}},\ }in\
  \href@noop {} {\emph {\bibinfo {booktitle} {COLT}}},\ \bibinfo {editor}
  {edited by\ \bibinfo {editor} {\bibfnamefont {A.~T.}\ \bibnamefont {Kalai}}\
  and\ \bibinfo {editor} {\bibfnamefont {M.}~\bibnamefont {Mohri}}}\ (\bibinfo
  {publisher} {Omnipress},\ \bibinfo {year} {2010})\ pp.\ \bibinfo {pages}
  {270--281}\BibitemShut {NoStop}%
\bibitem [{\citenamefont {van Laarhoven}\ and\ \citenamefont
  {Marchiori}(2014)}]{vanLaarhoven2014axioms}%
  \BibitemOpen
  \bibfield  {author} {\bibinfo {author} {\bibfnamefont {T.}~\bibnamefont {van
  Laarhoven}}\ and\ \bibinfo {author} {\bibfnamefont {E.}~\bibnamefont
  {Marchiori}},\ }\href@noop {} {\bibfield  {journal} {\bibinfo  {journal}
  {Journal of Machine Learning Research}\ }\textbf {\bibinfo {volume} {15}},\
  \bibinfo {pages} {193} (\bibinfo {year} {2014})}\BibitemShut {NoStop}%
\bibitem [{\citenamefont {Catral}\ \emph {et~al.}(2004)\citenamefont {Catral},
  \citenamefont {Han}, \citenamefont {Neumann},\ and\ \citenamefont
  {Plemmons}}]{Catral04}%
  \BibitemOpen
  \bibfield  {author} {\bibinfo {author} {\bibfnamefont {M.}~\bibnamefont
  {Catral}}, \bibinfo {author} {\bibfnamefont {L.}~\bibnamefont {Han}},
  \bibinfo {author} {\bibfnamefont {M.}~\bibnamefont {Neumann}}, \ and\
  \bibinfo {author} {\bibfnamefont {R.~J.}\ \bibnamefont {Plemmons}},\ }in\
  \href@noop {} {\emph {\bibinfo {booktitle} {Special Issue on Positivity in
  Linear Algebra}}}\ (\bibinfo {year} {2004})\ pp.\ \bibinfo {pages}
  {107--126}\BibitemShut {NoStop}%
\bibitem [{\citenamefont {Ding}\ \emph {et~al.}(2005)\citenamefont {Ding},
  \citenamefont {He},\ and\ \citenamefont {Simon}}]{Ding05onthe}%
  \BibitemOpen
  \bibfield  {author} {\bibinfo {author} {\bibfnamefont {C.}~\bibnamefont
  {Ding}}, \bibinfo {author} {\bibfnamefont {X.}~\bibnamefont {He}}, \ and\
  \bibinfo {author} {\bibfnamefont {H.~D.}\ \bibnamefont {Simon}},\ }in\
  \href@noop {} {\emph {\bibinfo {booktitle} {in SIAM International Conference
  on Data Mining}}}\ (\bibinfo {year} {2005})\BibitemShut {NoStop}%
\bibitem [{\citenamefont {Ding}\ \emph {et~al.}(2006)\citenamefont {Ding},
  \citenamefont {Li}, \citenamefont {Peng},\ and\ \citenamefont
  {Park}}]{Ding2006}%
  \BibitemOpen
  \bibfield  {author} {\bibinfo {author} {\bibfnamefont {C.}~\bibnamefont
  {Ding}}, \bibinfo {author} {\bibfnamefont {T.}~\bibnamefont {Li}}, \bibinfo
  {author} {\bibfnamefont {W.}~\bibnamefont {Peng}}, \ and\ \bibinfo {author}
  {\bibfnamefont {H.}~\bibnamefont {Park}},\ }in\ \href@noop {} {\emph
  {\bibinfo {booktitle} {Proceedings of the 12th ACM SIGKDD International
  Conference on Knowledge Discovery and Data Mining}}},\ \bibinfo {series and
  number} {KDD '06}\ (\bibinfo  {publisher} {ACM},\ \bibinfo {year} {2006})\
  pp.\ \bibinfo {pages} {126--135}\BibitemShut {NoStop}%
\bibitem [{\citenamefont {Palla}\ \emph {et~al.}(2005)\citenamefont {Palla},
  \citenamefont {Derenyi}, \citenamefont {Farkas},\ and\ \citenamefont
  {Vicsek}}]{Palla2005-clique-percolation}%
  \BibitemOpen
  \bibfield  {author} {\bibinfo {author} {\bibfnamefont {G.}~\bibnamefont
  {Palla}}, \bibinfo {author} {\bibfnamefont {I.}~\bibnamefont {Derenyi}},
  \bibinfo {author} {\bibfnamefont {I.}~\bibnamefont {Farkas}}, \ and\ \bibinfo
  {author} {\bibfnamefont {T.}~\bibnamefont {Vicsek}},\ }\href@noop {}
  {\bibfield  {journal} {\bibinfo  {journal} {Nature}\ }\textbf {\bibinfo
  {volume} {435}},\ \bibinfo {pages} {814} (\bibinfo {year}
  {2005})}\BibitemShut {NoStop}%
\bibitem [{\citenamefont {Ackerman}\ and\ \citenamefont
  {Ben-David}(2013)}]{Ackerman2013}%
  \BibitemOpen
  \bibfield  {author} {\bibinfo {author} {\bibfnamefont {M.}~\bibnamefont
  {Ackerman}}\ and\ \bibinfo {author} {\bibfnamefont {S.}~\bibnamefont
  {Ben-David}},\ }\href@noop {} {\bibfield  {journal} {\bibinfo  {journal}
  {Journal of Machine Learning Research}\ } (\bibinfo {year}
  {2013})}\BibitemShut {NoStop}%
\bibitem [{\citenamefont {Tan}\ and\ \citenamefont {Févotte}(2009)}]{Tan09}%
  \BibitemOpen
  \bibfield  {author} {\bibinfo {author} {\bibfnamefont {V.~Y.~F.}\
  \bibnamefont {Tan}}\ and\ \bibinfo {author} {\bibfnamefont {C.}~\bibnamefont
  {Févotte}},\ }in\ \href@noop {} {\emph {\bibinfo {booktitle} {Proc.~Workshop
  on Signal Processing with Adaptative Sparse Structured Representations
  (SPARS)}}}\ (\bibinfo {address} {St-Malo, France},\ \bibinfo {year}
  {2009})\BibitemShut {NoStop}%
\bibitem [{\citenamefont {Cemgil}(2009)}]{Cemgil2009BayesianNMF}%
  \BibitemOpen
  \bibfield  {author} {\bibinfo {author} {\bibfnamefont {A.~T.}\ \bibnamefont
  {Cemgil}},\ }\href@noop {} {\bibfield  {journal} {\bibinfo  {journal}
  {Intell. Neuroscience}\ }\textbf {\bibinfo {volume} {2009}},\ \bibinfo
  {pages} {4:1} (\bibinfo {year} {2009})}\BibitemShut {NoStop}%
\bibitem [{Note1()}]{Note1}%
  \BibitemOpen
  \bibinfo {note} {$S$ should be seen as a multiset, since multiple clusters
  can have the same support.}\BibitemShut {Stop}%
\bibitem [{\citenamefont {Pitman}(2002)}]{Pitman2002}%
  \BibitemOpen
  \bibfield  {author} {\bibinfo {author} {\bibfnamefont {J.}~\bibnamefont
  {Pitman}},\ }\href@noop {} {\bibfield  {journal} {\bibinfo  {journal}
  {Lecture Notes for St. Flour Summer School}\ } (\bibinfo {year}
  {2002})}\BibitemShut {NoStop}%
\bibitem [{\citenamefont {Lee}\ and\ \citenamefont
  {Seung}(2000)}]{LeeSeung2000algorithmsForNMF}%
  \BibitemOpen
  \bibfield  {author} {\bibinfo {author} {\bibfnamefont {D.~D.}\ \bibnamefont
  {Lee}}\ and\ \bibinfo {author} {\bibfnamefont {H.~S.}\ \bibnamefont
  {Seung}},\ }in\ \href@noop {} {\emph {\bibinfo {booktitle} {In NIPS}}}\
  (\bibinfo  {publisher} {MIT Press},\ \bibinfo {year} {2000})\ pp.\ \bibinfo
  {pages} {556--562}\BibitemShut {NoStop}%
\bibitem [{\citenamefont {Radicchi}\ \emph {et~al.}(2004)\citenamefont
  {Radicchi}, \citenamefont {Castellano}, \citenamefont {Cecconi},
  \citenamefont {Loreto},\ and\ \citenamefont {Parisi}}]{Radicchi2004}%
  \BibitemOpen
  \bibfield  {author} {\bibinfo {author} {\bibfnamefont {F.}~\bibnamefont
  {Radicchi}}, \bibinfo {author} {\bibfnamefont {C.}~\bibnamefont
  {Castellano}}, \bibinfo {author} {\bibfnamefont {F.}~\bibnamefont {Cecconi}},
  \bibinfo {author} {\bibfnamefont {V.}~\bibnamefont {Loreto}}, \ and\ \bibinfo
  {author} {\bibfnamefont {D.}~\bibnamefont {Parisi}},\ }\href@noop {}
  {\bibfield  {journal} {\bibinfo  {journal} {Proc. Natl. Acad. Sci. USA}\
  }\textbf {\bibinfo {volume} {101}},\ \bibinfo {pages} {2658} (\bibinfo {year}
  {2004})}\BibitemShut {NoStop}%
\end{thebibliography}%
